%% file: ms.tex
\documentclass[nohyperref]{article}

\usepackage{microtype}
\usepackage{graphicx}
\usepackage{subfigure}
\usepackage{booktabs} 

\usepackage{hyperref}
\usepackage{microtype}
\usepackage{booktabs} 
\usepackage{bm,xpatch}
\usepackage{amsfonts}
\usepackage{xcolor}
\usepackage{fp, tikz}
\usepackage{multirow}
\usepackage{interval}
\usepackage{capt-of}
\usepackage{comment}

\newcommand{\X}{\mathcal{X}}
\newcommand{\y}{\bm{y}}
\newcommand{\K}{\bm{K}}
\newcommand{\Id}{\bm{I}}
\newcommand{\transp}{^{\top}}

\makeatletter
\xpatchcmd{\proof}{\topsep0\p@\@plus0\p@\relax}{}{}{}
\makeatother

\usepackage[accepted]{icml2022}

\usepackage{amsmath}
\usepackage{amssymb}
\usepackage{mathtools}
\usepackage{amsthm}

\usepackage[capitalize,noabbrev]{cleveref}

\theoremstyle{plain}
\newtheorem{theorem}{Theorem}[section]

\newtheorem{lemma}[theorem]{Lemma}

\theoremstyle{definition}

\newtheorem{assumption}[theorem]{Assumption}
\theoremstyle{remark}
\newtheorem{remark}[theorem]{Remark}

\usepackage[disable,textsize=tiny]{todonotes}

\icmltitlerunning{Gaussian Process Uniform Error Bounds with Unknown Hyperparameters}

\begin{document}

\twocolumn[
\icmltitle{Gaussian Process Uniform Error Bounds with \\ Unknown Hyperparameters for Safety-Critical Applications}




\begin{icmlauthorlist}
\icmlauthor{Alexandre Capone}{tum}
\icmlauthor{Armin Lederer}{tum}
\icmlauthor{Sandra Hirche}{tum}
\end{icmlauthorlist}

\icmlaffiliation{tum}{TUM School of Computation, Information and Technology, Technical University of Munich, Munich, Germany}

\icmlcorrespondingauthor{Alexandre Capone}{alexandre.capone@tum.de}

\icmlkeywords{Machine Learning, ICML}

\vskip 0.3in
]



\printAffiliationsAndNotice{}  

\begin{abstract}
	Gaussian processes have become a promising tool for various safety-critical settings, since the posterior variance can be used to directly estimate the model error and quantify risk. However, state-of-the-art techniques for safety-critical settings hinge on the assumption that the kernel hyperparameters are known, which does not apply in general. To mitigate this, we introduce robust Gaussian process uniform error bounds in settings with unknown hyperparameters. Our approach computes a confidence region in the space of hyperparameters, which enables us to obtain a probabilistic upper bound for the model error of a Gaussian process with arbitrary hyperparameters. We do not require to know any bounds for the hyperparameters a priori, which is an assumption commonly found in related work. Instead, we are able to derive bounds from data in an intuitive fashion. We additionally employ the proposed technique to derive performance guarantees for a class of learning-based control problems. Experiments show that the bound performs significantly better than vanilla and fully Bayesian Gaussian processes.
\end{abstract}

\section{Introduction}
\label{section:introduction}

\noindent Gaussian processes (GPs) have become an often-used tool for regression due to their flexibility and good generalization properties. In addition to being successful at approximating unknown functions, GPs also come equipped with a measure of model uncertainty in the form of the posterior variance \cite{Rasmussen2006}. This quantity has shown promising results for estimating the error between the posterior mean and the underlying function \cite{Srinivas2012, chowdhury2017kernelized,lederer2019uniform, maddalena2021deterministic,sun2021uncertain}. 
Due to this characteristic, GPs have become particularly interesting for safety-critical settings, i.e., whenever safety or performance constraints need to be considered during decision-making. Some examples include controller tuning \cite{Capone2019BacksteppingFP,pmlr-v120-lederer20a}, estimating safe operating regions \cite{berkenkamp2017safe}, as well as recommendation engines \cite{pmlr-v37-sui15}.

%

\begin{figure}[t]
	\centering
	\includegraphics[scale=0.3,trim={0.5cm 1.2cm 0 0},clip]{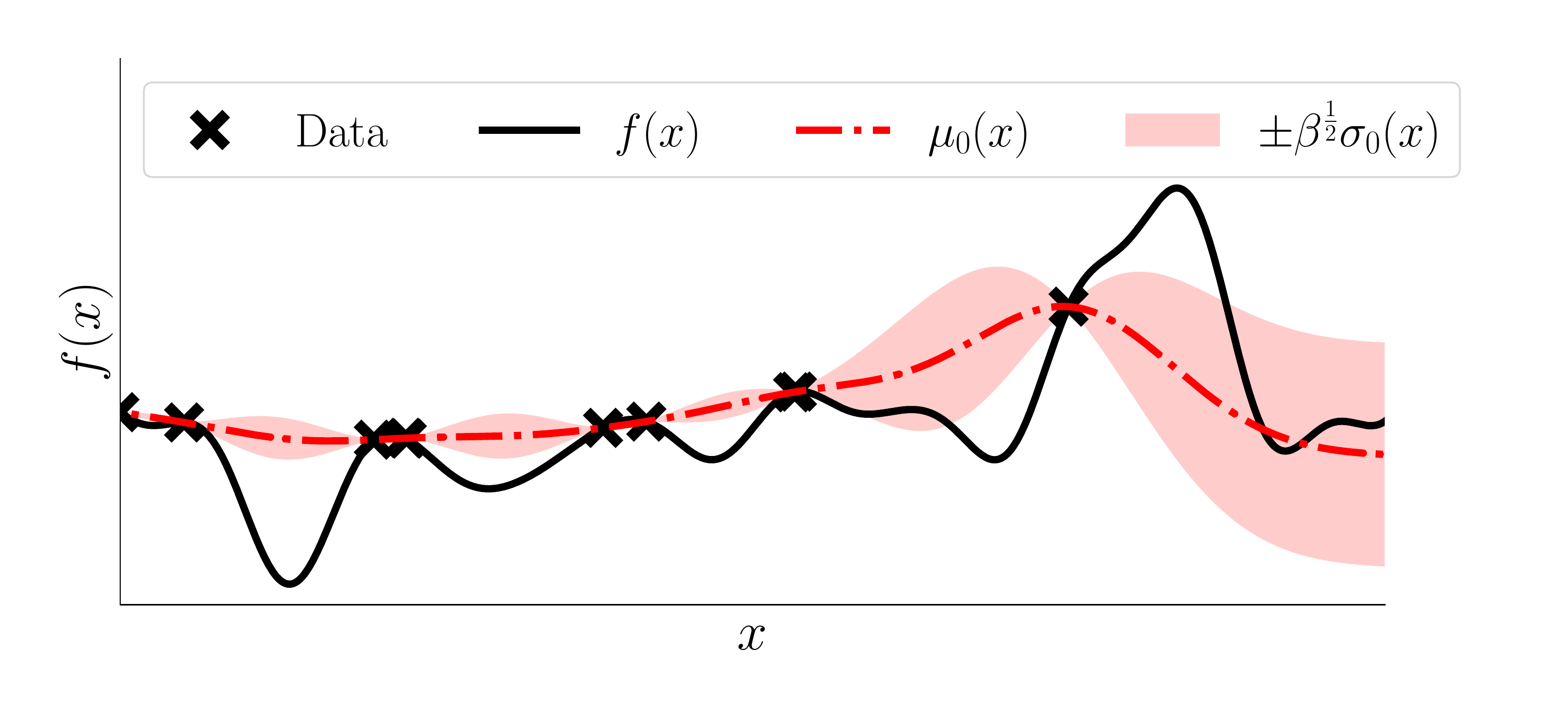}
	\caption{Overconfident uniform error bounds. The solid black line represents the unknown function, the dash-dotted line the GP mean. Crosses correspond to measurement data, the pink shaded area represents the estimated regression error. The estimated regression error increases too slowly away from the data due to the high lengthscale value and low signal variance. 
		As a result, most of $f(x)$ in regions with no data is not captured by the estimated regression error.} 
	\label{fig:overconfidentbounds}
\end{figure}

Uniform error bounds for Gaussian process regression are typically obtained by scaling the posterior standard deviation, where the scaling factor depends on the number of data, the input domain, as well as regularity assumptions on the unknown function \citep{Srinivas2012,chowdhury2017kernelized,maddalena2021deterministic,lederer2019uniform}.
The resulting error bound is generally of a probabilistic nature, i.e., it holds with high probability, with few exceptions \cite{maddalena2021deterministic,Wu1993}. 
In \citet{lederer2019uniform}, an error bound is derived under the assumption that the true function is drawn from a GP. The bound is achieved by estimating the error on a finite grid over the input space, and then extending the bound to the whole input space using Lipschitz properties of the GP. 
In \citet{Srinivas2012,chowdhury2017kernelized,maddalena2021deterministic}, error bounds are derived under the assumption that the underlying function belongs to a reproducing kernel Hilbert space. Additionally, error bounds from methods related to GPs can be directly exploited under weak restrictions. Since radial basis function interpolation yields regressors that are identical to the GP posterior mean for noise-free training data \cite{Kanagawa2018}, deterministic error bounds can be proven using the theory of reproducing kernel Hilbert spaces \cite{Wu1993, Wendland2004}. Moreover, error bounds from regularized kernel regression, as derived in \citet{Mendelson2002,Shi2013,Dicker2017}, can be straightforwardly extended to GPs due to the equivalence of regularized kernel and GP regression under weak assumptions \cite{Rasmussen2006}.

While using the posterior variance to estimate the model error can prove efficient in practice, this type of approach hinges on a critical assumption that seldom holds in practice, namely that the GP hyperparameters have been specified accurately. This is because the error bounds are derived by using smoothness assumptions about the GP kernel 
to bound the model error. 
Hence, if the choice of hyperparameters is poor, then the posterior variance is typically not a good measure of the model error. 
This potentially leads to overconfident error bounds, rendering the model inapplicable in safety-critical settings, as illustrated in \Cref{fig:overconfidentbounds}. 
%

\textbf{Related work.} This paper addresses GP uniform error bounds under misspecified kernel hyperparameters. 
To mitigate misspecified error bounds and guarantee no regret in a Bayesian optimization setting, \citet{berkenkamp2019no} developed an approach that gradually decreases the lengthscales of a GP, yielding convergence of the optimization algorithm towards the maximum of an unknown function. However, this is not useful for safety-critical settings, where the posterior variance, 
which varies strongly with the hyperparameters, is used to estimate the model error and determine the risk associated with decisions. For Matérn kernels, \citet{tuo2020kriging,wang2020prediction} have developed robust error bounds based on the fill distance under a misspecified smoothness parameter. Furthermore, \citet{tuo2020kriging} provides a more general bound based on upper and lower bounds for the decay rate of the true function. \citet{fiedler2021practical} have derived error bounds for settings where the norm of the difference between the kernel used for regression and that of the reproducing kernel Hilbert space (RKHS) of the true function can be bounded. Similarly, \cite{beckers2018misspecified} provides an error bound when choosing a covariance function from a predefined set, under the assumption that the unknown function corresponds to a GP with a kernel from the same set. A drawback of these approaches is that they do not provide a principled or intuitive approach for obtaining hyperparameter bounds or candidate kernels.

\textbf{Our contribution.} In this work, we mitigate the risk of making a poor choice of hyperparameters by equipping a Gaussian process with an error bound that accounts for the lack of prior knowledge with regard to the hyperparameters. We additionally present a principled way of choosing robust uniform error bounds in a Bayesian setting without requiring any prior upper and lower bounds for the hyperparameters.  
The corresponding theoretical guarantees hold for a class of frequently used kernels, extending the applicability of Gaussian processes in safety-critical settings. 

\section{Gaussian Processes}
\label{section:background}
\noindent In this section, we briefly review GPs and then discuss how error bounds and hyperparameters are chosen in practice.

We use GPs for regression, where we aim to infer an unknown function $f: \X \rightarrow \mathbb{R}$ with $\X\subseteq \mathbb{R}^d$. To this end, we treat function values $f(\bm{x})$ as random variables, of which any subset is jointly normally distributed \cite{Rasmussen2006}. 
A Gaussian process is fully specified by a mean function $m:  \X \rightarrow \mathbb{R}$ and a positive-definite kernel $k_{\bm{\vartheta}}: \X \times \X \rightarrow \mathbb{R}$. In this paper, we set $m=0$ without loss of generality, and restrict ourselves to  stationary kernels with radially non-increasing Fourier transform. This is specified by the following assumption.
\begin{assumption}
	\label{assumption:kernelassumption}
	The kernel $k_{\bm{\vartheta}}(\cdot,\cdot)$ is of the form
	\[k_{\bm{\vartheta}}(\bm{x},\bm{x}') \coloneqq k\left(\left(\frac{x_1-x_1'}{\vartheta_{1}}, \ldots, \frac{x_d-x_d'}{\vartheta_{d}}\right)^{\top}\right), \]
	for some $k:\mathbb{R}^d\rightarrow \mathbb{R}$ with $k(\bm{0}) = 1$. Furthermore, $k(\cdot)$ has Fourier transform $\hat{k}(\bm{\omega})  \coloneqq \kappa(\lVert \bm{\omega}\rVert_2)$ for some non-increasing non-negative function $\kappa: \mathbb{R}_+ \rightarrow \mathbb{R}_+$. 
\end{assumption}Here $\bm{\vartheta} \coloneqq (\vartheta_1, \ldots, \vartheta_d)^{\top} \in \bm{\varTheta}$ denotes the kernel lengthscales, which scale the kernel inputs and specify the variability of the underlying function. A similar assumption can be found in other papers that analyze algorithms based on stationary kernels \cite{bull2011convergence,berkenkamp2019no}. 
When using this type of kernel, the similarity between any two function evaluations decreases with the weighted distance between the corresponding inputs. Many frequently encountered kernels satisfy these properties. Examples include the Gaussian and Matérn kernels, which are often employed because they satisfy the universal approximation property \cite{micchelli2006universal}. Note that we assume $k(\bm{0},\bm{0}) = 1$ only for the sake of simplicity, and all the results and tools presented in this paper can be straightforwardly extended to the more general case where the scaling factor of $k(\cdot,\cdot)$, i.e., the signal variance, is also a hyperparameter. This is discussed in \Cref{subsection:extension}.  

Given a set of $N$ (potentially noisy) measurements $\mathcal{D} = \{\bm{x}_i, y_i\}_{i=1}^N\coloneqq \{\bm{X},\bm{y}\}$, where $y_i = f(\bm{x}_i) + \varepsilon_i$, $\varepsilon_i \sim \mathcal{N}(0, \sigma_n^2)$ and $\bm{x}_i \in \X$, we are able to condition a Gaussian process on $\mathcal{D}$ to obtain the posterior distribution of $f(\bm{x}^*)$ at an arbitrary test point $\bm{x}^* \in \X$. The corresponding distribution is normal $f(\bm{x}^*)\sim \mathcal{N}(\mu_{\bm{\vartheta}}(\bm{x}^*) ,\sigma_{\bm{\vartheta}}^2(\bm{x}^*) )$ 
with mean and variance
\begin{align*}
	\begin{split}
		\mu_{\bm{\vartheta}}(\bm{x}^*) &\coloneqq  \bm{k}_{\bm{\vartheta}}(\bm{x}^*)\transp \left(\bm{K}_{\bm{\vartheta}} + \sigma_n^2\Id\right) 
		^{-1} \bm{y} \\
		\sigma^2_{\bm{\vartheta}}(\bm{x^*}) &\coloneqq k_{\bm{\vartheta}}(\bm{x}^*,\bm{x}^*) - \bm{k}_{\bm{\vartheta}}(\bm{x}^*)\transp \left(\bm{K}_{\bm{\vartheta}} + \sigma_n^2\Id\right) ^{-1} \bm{k}_{\bm{\vartheta}}(\bm{x}^*),
	\end{split}
\end{align*}
where {$\bm{k}_{\bm{\vartheta}}(\bm{x}^*) = (k_{\bm{\vartheta}}(\bm{x}_1, \bm{x}^*),\ldots, k_{\bm{\vartheta}}(\bm{x}_N, \bm{x}^*))^{\top}$, $\sigma_n^2$ denotes the noise variance, and the entries of the covariance matrix $\bm{K}_{\bm{\vartheta}}$ are given by $(\bm{K}_{\bm{\vartheta}})_{i,j} = k_{\bm{\vartheta}}(\bm{x}_i,\bm{x}_j)$}. 
In the case of stationary kernels, the posterior variance $\sigma^2_{\bm{\vartheta}}(\bm{x^*}) $ is typically low for test inputs $\bm{x^*}$ that are close to the measurement data $\mathcal{D}$ and vice versa.


\subsection{Choosing Hyperparameters}
\label{subsection:hyperparametertuning}

By far the most common technique for choosing the hyperparameters of a Gaussian process is maximizing the log marginal likelihood of the measurements given the hyperparameters
\begin{align}
	\label{eq:loglikelihood}
	\begin{split}
		\log p(\bm{y} \vert \bm{X}, \bm{\vartheta}) = &-\frac{1}{2}\log{\vert \K_{\bm{\vartheta}} + \sigma_n^2 \Id \vert -\frac{N}{2}\log(2\pi)} \\ &- \frac{1}{2}\y^\intercal \left(\K_{\bm{\vartheta}} +\sigma_n^2 \Id\right)^{-1}\y .
	\end{split}
\end{align}

By maximizing \eqref{eq:loglikelihood}, we obtain a trade-off between model complexity and data fit. Coupled with the fact that the gradient of \eqref{eq:loglikelihood} can be computed analytically, choosing hyperparameters in this manner yields many practical benefits. Furthermore, if the marginal likelihood is well peaked, then the true hyperparameters are likely to be situated near the selected ones. However, if this is not the case, then this approach can lead to overconfident hyperparameters, as the marginal likelihood can decrease slowly away from the maximum, implying that the lengthscales are potentially smaller than the ones obtained. 
This type of behavior is particularly frequent if little data has been observed, as both short and long lengthscales explain the data consistently \cite{Rasmussen2006}. Less common approaches for choosing the Gaussian process hyperparameters include the cross validation \cite{cressie2015statistics} and log-pseudo likelihood maximization \cite{sundararajan2001predictive}. In some settings, the hyperparameters are chosen based on prior knowledge about the system \cite{kirschner2019adaptive}. 

\section{Uniform Error Bounds for Unknown Hyperparameters}
\label{section:safegaussianprocesses}

\noindent We now introduce a modified version of standard GP error bounds that aims to overcome the limitations mentioned in the previous chapter. The proofs of all results stated here can be found in the appendix. 
%

In the Gaussian process regression literature, there are typically two different types of assumptions that can be made for analysis. On the one hand, the Bayesian case can be considered, where we assume the unknown function $f(\cdot)$ to be sampled from a Gaussian process \cite{sun2021uncertain}. On the other hand, the frequentist setting can be considered, where $f(\cdot)$ is assumed to have a bounded RKHS norm \cite{Srinivas2012,chowdhury2017kernelized} with respect to the kernel $k_{\bm{\vartheta}}(\cdot,\cdot)$. While some of the techniques presented in this paper can be useful both in a Bayesian and frequentist setting, the Bayesian paradigm provides a straightforward way of narrowing the set of candidate hyperparameters used to derive an error bound, whereas no such path is evident in the frequentist setting. Hence, we henceforth assume a Bayesian setting, as described in the following assumption.

\begin{assumption}[Bayesian setting]
	\label{assumption:bayesiansetting}
	The unknown function $f(\cdot)$ corresponds to a sample from a Gaussian process with kernel $k_{\bm{\vartheta}}(\cdot,\cdot)$, where the hyperparameters $\bm{\vartheta}$ are drawn from a hyperprior $p(\bm{\vartheta})$.
\end{assumption}

\begin{figure}[t!]
	\centering
	\includegraphics[scale=0.28,trim={1.2cm 1.1cm 0 0},clip]{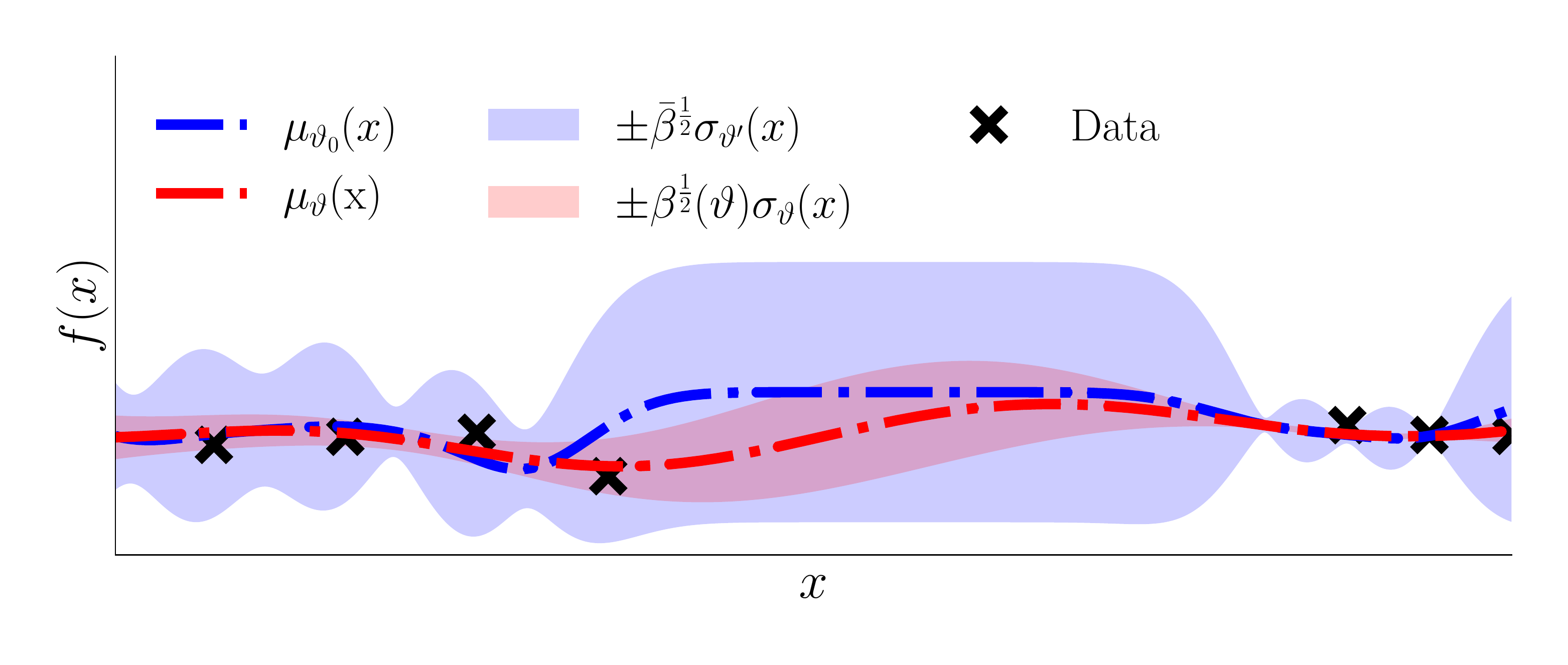}
	\caption{The uniform error bound obtained with bounding vectors $\bm{\vartheta}',\bm{\vartheta}''$ for a posterior mean with working hyperparameters $\bm{\vartheta}_0$, in blue/light blue, fully contains that of the Gaussian process with lengthscales $\bm{\vartheta}$, shown in red/light red, provided that $\bm{\vartheta}' \leq \bm{\vartheta},\bm{\vartheta}_0 \leq\bm{\vartheta}''$, where $\leq$ denotes component-wise inequality. This yields \Cref{theorem:errorwithknownbounds}.} 
	\label{fig:ToyProblemPredictionWithLearning}
\end{figure} 

This assumption is not very restrictive, and is often encountered in control and reinforcement learning settings \cite{deisenroth2015gaussian,kocijan2005dynamic,hewing2019cautious}. 
The choice of prior $p(\bm{\vartheta})$ can be based on any prior beliefs about the underlying function, e.g., Lipschitz continuity, which can be encoded into the prior using a chi-squared or uniform distribution.

We also assume to have a scaling function that specifies a uniform error bound for an arbitrary fixed vector of hyperparameters $\bm{\vartheta}$ with a prespecified probability $1-\rho$.

\begin{assumption}[Scaling function]
	\label{assumption:boundingfunctionbayes}
	For an arbitrary $\rho$ with $0<\rho<1$, there exists a known positive function $\beta: \mathbb{R}^d\rightarrow\mathbb{R}_+$, such that for any vector of lengthscales $\bm{\vartheta} \in \bm{\varTheta}$,
	\begin{align*}
		\text{P} \Big(\vert f(\bm{x}) - \mu_{\bm{\vartheta}}(\bm{x}) \vert \leq \beta^{\frac{1}{2}}(\bm{\vartheta}) \sigma_{\bm{\vartheta}}(\bm{x} ) \ \  \forall \bm{x} \in \X\Big) \geq 1-\rho
	\end{align*}
	holds, where $f(\bm{x})$ denotes a sample from a Gaussian process with prior mean zero and kernel $k_{\bm{\vartheta}}(\cdot,\cdot)$, conditioned on measurement data $\mathcal{D}$.
\end{assumption}
%

Error bounds for fixed hyperparameters $\bm{\vartheta}$ in the Bayesian case can be derived by assuming that the input space is compact and that the kernel satisfies some regularity requirements \cite{Srinivas2012}. Improved error bounds can be obtained if we are only interested in samples $f(\bm{x})$ that satisfy a predefined Lipschitz continuity requirement \cite{lederer2019uniform,sun2021uncertain}. Additionally, if we assume the input space to be finite, then $\beta(\cdot)$ is a constant, i.e., does not depend on $\bm{\vartheta}$ \cite{Srinivas2012}. 

Our approach is then based on the following result.
\begin{lemma}
	\label{lemma:sigmasdecreasewithlengthscale}
	Let $k(\cdot)$ be a kernel, $\bm{\vartheta}',\bm{\vartheta} '',\bm{\vartheta} \in \bm{\varTheta}$ vectors of lengthscales with  $\bm{\vartheta}'\leq \bm{\vartheta}\leq \bm{\vartheta}''$, where $\leq$ denotes component-wise inequality, and let $\mathcal{D}$ denote a measurement data set of an unknown function $f(\cdot)$. Furthermore, choose $\gamma>0$ with $\gamma^2 = \prod_{i=1}^{d} \frac{{\vartheta}''_i}{\vartheta'_i}$. 
	Then
	\begin{align*}
		\sigma_{\bm{\vartheta}}(\bm{x}) \leq \gamma \sigma_{\bm{\vartheta}'}(\bm{x})
	\end{align*}
	holds for all $\bm{x}\in \X$.
\end{lemma}

\begin{remark}
	\label{remark:smallergamma}
	 \Cref{lemma:sigmasdecreasewithlengthscale} is very general, which might lead to conservative scaling factors $\gamma$. In some settings, however, the inequality in \Cref{lemma:sigmasdecreasewithlengthscale} also holds for smaller values of $\gamma$. This is the case, e.g., for the commonly employed squared-exponential kernel, as illustrated in \Cref{section:results}.
\end{remark}

In essence, \Cref{lemma:sigmasdecreasewithlengthscale} states that the scaled posterior variance decreases with the lengthscales. As a direct consequence, if a pair of bounding hyperparameters $\bm{\vartheta}',\bm{\vartheta}''$ is available, we can use \Cref{lemma:sigmasdecreasewithlengthscale} to construct an error bound that contains all error bounds corresponding to lengthscales within then interval $\bm{\vartheta}' \leq \bm{\vartheta} \leq \bm{\vartheta}''$. This is specified in the following result.

%

\begin{theorem}
	\label{theorem:errorwithknownbounds}
	Let \Cref{assumption:boundingfunctionbayes} and \Cref{assumption:kernelassumption} hold, and let $\rho$ and $\beta(\cdot)$ be as in \Cref{assumption:boundingfunctionbayes}. Let $\bm{\vartheta}',\bm{\vartheta}'',\bm{\vartheta} \in \bm{\varTheta}$ be vectors of lengthscales with $\bm{\vartheta}' \leq \bm{\vartheta} \leq \bm{\vartheta}''$, where $\leq$ denotes component-wise inequality,  let $\mathcal{D}$ be a measurement data set, and let $f(\bm{x})$ be a sample from a Gaussian process with mean zero and kernel $k_{\bm{\vartheta}}(\cdot,\cdot)$, conditioned on $\mathcal{D}$. Furthermore, let $\mu_{\bm{\vartheta}_0}(\cdot)$ denote the posterior mean for arbitrary lengthscales $\bm{\vartheta}_0$ with $\bm{\vartheta}'\leq \bm{\vartheta}_0 \leq\bm{\vartheta}''$, and define 
	\begin{align}
		\label{eq:barbeta}
		\bar{\beta} = \gamma^2\left(\max\limits_{\bm{\vartheta}'\leq\bm{\vartheta}\leq \bm{\vartheta}''}\beta^{\frac{1}{2}}(\bm{\vartheta})  +  \frac{2 \lVert\bm{y}\rVert_2}{\sigma_n} \right)^2.
	\end{align} 
It then holds that
	\begin{align*}
		\text{P} \left(\vert f(\bm{x}) - \mu_{\bm{\vartheta}_0}(\bm{x}) \vert \leq {\bar{\beta}}^{\frac{1}{2}}  \sigma_{\bm{\vartheta}'} (\bm{x})\quad \forall \bm{x} \in \X \right) \geq 1-\rho.
	\end{align*}
\end{theorem}

\Cref{theorem:errorwithknownbounds} implies that we can obtain a uniform error bound using the scaled posterior variance corresponding to the smallest vector of lengthscales $\bm{\vartheta}'$. 
\begin{remark}
The term $2\sigma_n^{-1}\lVert\bm{y}\rVert_2$ in \eqref{eq:barbeta} originates from the discrepancy $\vert\mu_{\bm{\vartheta}}(\bm{x})- \mu_{\bm{\vartheta}_0}(\bm{x})\vert$ between the posterior means using the working lengthscales $\bm{\vartheta}_0$ and $\bm{\vartheta}$ within the bounding lengthscales $\bm{\vartheta}'\leq \bm{\vartheta}\leq \bm{\vartheta}''$. In practice, if this discrepancy is determined to be small, e.g., by evaluating the discrepancy $\vert\mu_{\bm{\vartheta}}(\bm{x})- \mu_{\bm{\vartheta}_0}(\bm{x})\vert$ for different $\bm{\vartheta}$ and test inputs $\bm{x}$, then $2\sigma_n^{-1}\lVert\bm{y}\rVert_2$ can be substituted by smaller values.
\end{remark} 
Note that the probabilistic nature of \Cref{theorem:errorwithknownbounds} stems from \Cref{assumption:boundingfunctionbayes} also being probabilistic, and no additional uncertainty is introduced when deriving \Cref{theorem:errorwithknownbounds}. In other words, the confidence region generated by the robust uniform error bound fully contains all confidence regions corresponding to $\bm{\vartheta}$ with $\bm{\vartheta}'\leq \bm{\vartheta} \leq \bm{\vartheta}''$. This is illustrated in \Cref{fig:ToyProblemPredictionWithLearning}.

In order to be able to apply \Cref{theorem:errorwithknownbounds} in the more general setting where the bounding vectors $\bm{\vartheta}',\bm{\vartheta}''$ are not given a priori, the next step is to determine permissible $\bm{\vartheta}',\bm{\vartheta}''$. Since we are in a Bayesian scenario, an intuitive approach is to choose the bounding lengthscales $\bm{\vartheta}',\bm{\vartheta}''$ such that they form a $1-\delta$ confidence interval for some $\delta \in (0,1)$. Note that deriving a similar technique in the frequentist setting, e.g., as in \citet{chowdhury2017kernelized} or \citet{Srinivas2012}, is not straightforward since the corresponding tools are not directly compatible with a distribution over the hyperparameters.

Formally, our approach consists of obtaining a pair of hyperparameters $\bm{\vartheta}',\bm{\vartheta}''$ that lies within the set
\begin{align}
	\label{eq:safeset}
	\mathcal{P}_{\delta} = \left\{(\bm{\vartheta}', \bm{\vartheta}'') \in \bm{\varTheta}^2 \ \Bigg\vert \int\limits_{ \bm{\vartheta}' \leq \bm{\vartheta}\leq \bm{\vartheta}''} \!\!\! p(\bm{\vartheta} \vert \mathcal{D} )d\bm{\vartheta} \geq 1- \delta \right\},
\end{align}
where the posterior over the lengthscales $\bm{\vartheta}$ given the data $\mathcal{D}=\{\bm{X},\bm{y}\}$ is computed using Bayes' rule as
\begin{align}
	p(\bm{\vartheta} \vert \bm{y}, \bm{X} ) = \frac{p(\bm{y} \vert \bm{X}, \bm{\vartheta}) p(\bm{\vartheta})}{p(\bm{y} \vert \bm{X})},
\end{align}
and the posterior $p(\bm{y} \vert \bm{X}, \bm{\vartheta})$ is computed similarly to \eqref{eq:loglikelihood}. The normalizing factor $p(\bm{y} \vert \bm{X})$ is then given by
\begin{align}
	\label{eq:normalizingfactor}
	p(\bm{y} \vert \bm{X}) = \int\limits_{\bm{\varTheta} } p(\bm{y} \vert \bm{X}, \bm{\vartheta}) p(\bm{\vartheta}) d\bm{\vartheta} .
\end{align}

By choosing the bounding lengthscales $\bm{\vartheta}',\bm{\vartheta}''
$ in this fashion, a vector of lengthscales $\bm{\vartheta}$ sampled from the posterior distribution $p(\bm{\vartheta}\vert \mathcal{D})$ lies within the interval $\bm{\vartheta}'\leq \bm{\vartheta}\leq\bm{\vartheta}''
$ with high probability.
We are then able to estimate the error in the setting with unknown hyperparameters by applying \Cref{theorem:errorwithknownbounds}, obtaining the following result.

\begin{figure}[t!]
	\centering
	\includegraphics[scale=0.3]{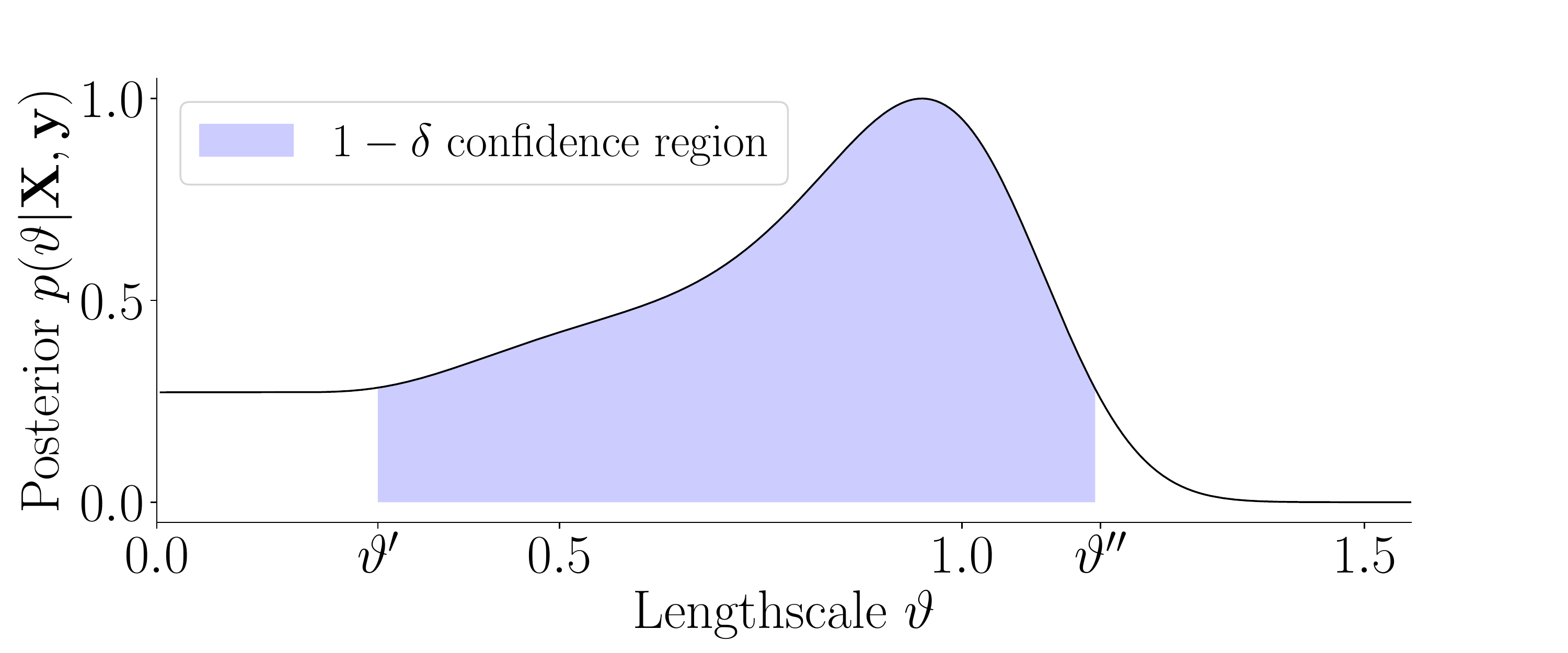}
	\caption{Posterior probability distribution $p(\vartheta \vert \bm{X},\bm{y})$ given data $\mathcal{D}=\{\bm{X},\bm{y}\}$ and confidence region of $1-\delta$ generated by bounding lengthscales $(\vartheta',\vartheta'')$. By applying \Cref{theorem:errorwithknownbounds}, we can use the pair $(\vartheta',\vartheta'')$ to obtain a robust error bound. This yields \Cref{theorem:mainresult}. Note that the posterior is poorly peaked, which indicates that the lengthscales obtained via log-likelihood maximization are potentially too high, leading to overconfident error bounds.} 
\end{figure}  

\begin{theorem}
	\label{theorem:mainresult}
	Let \Cref{assumption:boundingfunctionbayes} and \Cref{assumption:kernelassumption} hold. Let $\bm{\vartheta}',\bm{\vartheta}_0,\bm{\vartheta}'' \in \bm{\varTheta}$ be vectors of lengthscales with $(\bm{\vartheta}',\bm{\vartheta}'') \in \mathcal{P}_{\delta}$ and choose $\bar{\beta}$ as in \Cref{theorem:errorwithknownbounds}. Furthermore, let $\bm{\vartheta}\sim p(\bm{\vartheta}\vert\mathcal{D})$, and let $f(\cdot)$ denote a sample from a Gaussian process with kernel $k_{\bm{\vartheta}}(\cdot,\cdot)$ as specified by \Cref{assumption:bayesiansetting}.
	Then
	\begin{align}
		\label{eq:corollaryerrorbound}
		\begin{split}
			\vert f(\bm{x}) - \mu_{\bm{\vartheta}_0}(\bm{x}) \vert \leq   \bar{\beta}^{\frac{1}{2}}\sigma_{\bm{\vartheta}'}(\bm{x})
		\end{split}
	\end{align}
	holds for all $\bm{x} \in \X$ with probability $(1-\delta)(1-\rho)$.
\end{theorem}

\subsection{Extension to Noise and Signal Variance}
\label{subsection:extension}
So far, we discussed the setting where the signal and noise variance are constant and known. However, the proposed tools extend straightforwardly to the setting where they are also unknown, provided that a corresponding prior distribution is also available. This is because the posterior variance $\sigma^2_{\bm{\vartheta}}(\cdot)$ increases with the noise and signal variance, hence a result of the form of \Cref{theorem:mainresult} can be easily obtained. 

\subsection{Choosing Bounding Hyperparameters}
Typically, the set of bounding pairs $\mathcal{P}_{\delta}$ contains more than one element. Hence, \Cref{theorem:mainresult} allows us some flexibility when deriving the uniform error bound, since we can choose any pair $(\bm{\vartheta}',\bm{\vartheta}'')$ that lies within $\mathcal{P}_{\delta}$. In the following, we derive a heuristic that aims to approximate the smallest possible $1-\delta$ confidence region around the working hyperparameters $\bm{\vartheta}_0$, which would automatically meet the requirements to apply \Cref{theorem:mainresult}. 
%
%
Formally, this is achieved by 
solving the optimization problem
\begin{align}
	\label{eq:optimizationproblem}
	\begin{split}
		\min \limits_{\bm{\vartheta}',\bm{\vartheta}''\in \bm{\varTheta}} & \qquad \lVert \bm{\vartheta}''-\bm{\vartheta}'\rVert_2 \\
		\text{s.t.} & \int\limits_{\bm{\vartheta}' \leq \bm{\vartheta} \leq \bm{\vartheta}''} \!\!\! p(\bm{\vartheta} \vert \mathcal{D} )d\bm{\vartheta} \geq 1- \delta.
	\end{split}
\end{align}
In other words, we aim to choose the pair of bounding hyperparameters $(\bm{\vartheta}',\bm{\vartheta}'')\in \mathcal{P}_{\delta} $ that deviate the least from $\bm{\vartheta}_0$. 



\subsection{Discussion}
\label{subsection:discussion}

Although the integrals in \Cref{eq:safeset} and \eqref{eq:normalizingfactor} generally cannot be computed analytically, we can resort to different approximations. In low-dimensional settings, we can solve the integral expression using numerical integration. A further option is to employ Markov chain Monte Carlo (MCMC) methods \cite{Rasmussen2006}. Alternatively, we can apply approximate inference methods, such as Laplace's method or expectation propagation, which aim to approximate the posterior $p(\bm{\vartheta} \vert \bm{y}, \bm{X} )$ with a normal distribution. The latter methods are particularly well suited for settings with a high number of data points, as they do not require a high number of evaluations of the posterior distribution. Note that this is an advantage of the proposed technique compared to fully Bayesian GPs, which typically require some form of MCMC approach. Moreover, after computing the bounding hyperparameters $\bm{\vartheta}',\bm{\vartheta}''$, the computational complexity of evaluating the error bound is only twice that of a standard GP, whereas in a fully Bayesian GP a potentially high number of GPs has to be evaluated per prediction. A further advantage compared to fully Bayesian GPs is that, in practice, it is reasonable to expect better error estimates, particularly if we choose a low risk parameter $\delta$ and a sufficiently smooth prior $p(\bm{\vartheta})$. This is because the marginal likelihood is typically well behaved in the hyperparameter space, and the fully Bayesian posterior mean and variance will lie within the uniform error bound computed with \Cref{theorem:mainresult}. This is supported by the experimental results in \Cref{section:results}. We note, however, that the proposed approach can also be combined with a fully Bayesian GP, i.e., by using the fully Bayesian posterior mean only for regression, and the presented techniques to bound the regression error.

It is also worth noting that there is no direct analogy between the approach presented in this paper and robust uniform error bounds in the frequentist setting, i.e., where the unknown function $f(\cdot)$ is assumed to be fixed with a bounded RKHS norm. We believe this to be a strong argument in favor of employing a Bayesian perspective instead of a frequentist one when performing regression.

\section{Control with Performance Guarantees}
\label{section:backstepping}

\noindent\Cref{theorem:mainresult} can be employed straightforwardly to derive safety guarantees in different settings. In the following, we show how to apply it in a learning-based control setting with a commonly encountered system structure.

\subsection{Control Problem}
\label{subsection:controlproblem}

We consider an $m$-dimensional dynamical system that obeys the frequently encountered strict-feedback form \cite{sabanovic1993sliding,krstic1995jetengine,Kwan2000}:
\begin{align}
	\label{eq:strictfeedbackform}
	\begin{split}
		\dot{x}_1 &= f_1(x_1) + g(x_1)x_2 \\
		\dot{x}_2 &= f_2(x_1,x_2) + g(x_1,x_2)x_3 \\
		\vphantom{A^{A^a}}\smash[tb]{\vdots} \\
		\dot{x}_m &= f_m(x_1,\ldots,x_m) + g(x_1,\ldots, x_m)u, \\
	\end{split}
\end{align}
where $x_1,\ldots,x_m \in \mathbb{R}$ and $u\in \mathbb{R}$ denote the system's states and control input, respectively. The functions $f_i: \mathbb{R}^i \rightarrow \mathbb{R}$ are unknown and modeled using GPs, whereas we assume to know $g_i: \mathbb{R}^i \rightarrow \mathbb{R}$. These assumptions are common in this setting \cite{Kwan2000,Capone2019BacksteppingFP}.

Our goal is to design a control law $u$ that steers the subsystem $\dot{x}_1 = f_1(x_1) + g(x_1)x_2 $ towards a desired time-dependent trajectory $x_d(t)$. In other words, we aim to reduce the norm of the time-dependent error $e_1(t) = x_1(t)-x_d(t)$. In the following, we assume that $x_d(t)$ is $m$ times continuously differentiable, and that all its derivatives are bounded, which is not a restrictive assumption. For the sake of simplicity, we henceforth use the notation $f_i \coloneqq f_i(x_1,\ldots,x_i)$ and $g_i \coloneqq g_i(x_1,\ldots, x_i)$.

In order to provide performance guarantees for \eqref{eq:strictfeedbackform}, we require some additional smoothness assumptions with respect to the kernels used to model the functions $f_i$. This is expressed formally in the following.
\begin{assumption}
	\label{assumption:kernelassumptionscontrol}
	Each function $f_i$ is drawn from a Gaussian process with $m-i$ times differentiable kernel $k_{\bm{\vartheta}_i}(\cdot,\cdot)$.
\end{assumption}
\Cref{assumption:kernelassumptionscontrol} implies that the functions $f_i$ are $m-i$ times differentiable, which is not a restrictive assumption, as it applies for many systems of the form given by \eqref{eq:strictfeedbackform}, e.g., robotic manipulators, jet engines, and induction motors \cite{Kwan2000}.

\subsection{Backstepping Control}
\label{subsection:backsteppingcontrol}
In order to track the desired trajectory $x_d(t)$, we employ a backstepping technique similar to the one proposed in \citet{Capone2019BacksteppingFP}. The idea behind backstepping is to recursively design fictitious control inputs $ x_{i+1,d}$ for each subsystem $\dot{x}_i = f_i + g_i x_{i+1,d}$, which we then aim to track by means of the true control input $u$. For more details, the reader is referred to \cite{krstic1995nonlinear}. 

To model the unknown functions $f_i(\cdot)$, we assume to have $N$ noisy measurements $\{x_{1\ldots i,j}, f_{i,j}\}_{j=1}^{N}$. The control input $u$ can then be computed recursively using
\begin{align}
	\label{eq:backsteppingcontrolinput}
	\begin{split}
		x_{2,d} &= g_1^{-1}\left(-\mu_1 + \dot{x}_d - C_1 e_1\right) \\
		\vphantom{A^{A^a}}\smash[tb]{\vdots} \\
		x_{i,d} &= g_{i-1}^{-1}\left(-\mu_{i-1} + \dot{x}_{i-1,d} - C_{i-2} e_{i-2}\right) \\
		\vphantom{A^{A^a}}\smash[tb]{\vdots} \\
		u &= g_m^{-1}\left(-\mu_m + \dot{x}_{m,d} - e_m - g_{m-1}e_{m-1}\right),
	\end{split}
\end{align}
where $\mu_i$ denotes the posterior mean of the $i$-th Gaussian process, $e_i = x_i - x_{i,d}$ the tracking error of the $i$-th subsystem, and $C_i: \mathbb{R}^i \rightarrow \mathbb{R}$ are state-dependent control gains.

The stability analysis of the closed-loop system is carried out using Lyapunov theory. By using a quadratic Lyapunov function, it is straightforward to show that there exists an ultimate upper bound $\xi$ for the tracking error $e_d(t)$, i.e., $e_d(t) \leq \xi$ holds for large enough $t$, where $\xi>0$ depends on the control gains $C_i$ \citep[Lemma 2]{Capone2019BacksteppingFP}. By using this result and \Cref{theorem:mainresult}, we are able to determine adaptive control gains $C_i$ that yield a tracking error smaller or equal than a predetermined value $\xi_{\text{des}}$, which is crucial in safety-critical applications. This is formally stated in the following result.

\begin{theorem}
	\label{theorem:backsteppingcontrol}
	Let \Cref{assumption:kernelassumption} and \Cref{assumption:kernelassumptionscontrol} hold, and let the control input $u$ be given by \eqref{eq:backsteppingcontrolinput}. Furthermore, let $(\bm{\vartheta}_1',\bm{\vartheta}''_1),\ldots,(\bm{\vartheta}'_m,\bm{\vartheta}''_m)$ be bounding hyperparameters for the subsystems $1,\ldots,m$, obtained with \eqref{eq:optimizationproblem} and risk parameter $\delta \in (0,1)$, and let $\beta_i(\cdot)$ be the corresponding scaling functions. Choose $\xi_{\text{des}}>0$ and let $\bar{\beta}_i$ be as in \Cref{theorem:errorwithknownbounds}, and choose each state-dependent control gain as 
	\begin{align}
		\label{eq:adaptivegains}
		C_i(\bm{x}) \coloneqq  \frac{1}{{\xi_{\text{des}}} } \sqrt{\sum\limits_{j=1}^m{\bar{\beta}_j} \sigma^2_{\bm{\vartheta}_j'}(\bm{x})},\end{align}
	where $\gamma_j$ are chosen as in \Cref{lemma:sigmasdecreasewithlengthscale}. Then, with probability $(1-\rho)^m(1-\delta)^m$, there exists a $T>0$, such that
	\begin{align}
		\Vert \bm{x}(t) - \bm{x}_d(t) \Vert_2 \leq \xi_{\text{des}}
	\end{align}
	holds for all $t>T$.
\end{theorem}

Hence, we can use \eqref{eq:adaptivegains} to obtain the control gains required for the control performance specification $\xi_{\text{des}}$.

\section{Results}
\label{section:results}

\begin{figure*}[t!]
	\centering
	\includegraphics[scale=0.3]{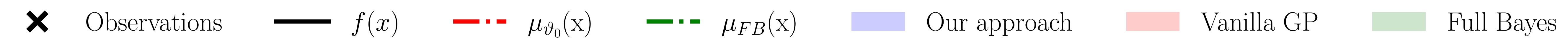}
	\subfigure[$N=2$ data points.]{\label{fig:robustgpn2}\includegraphics[width=.32\textwidth]{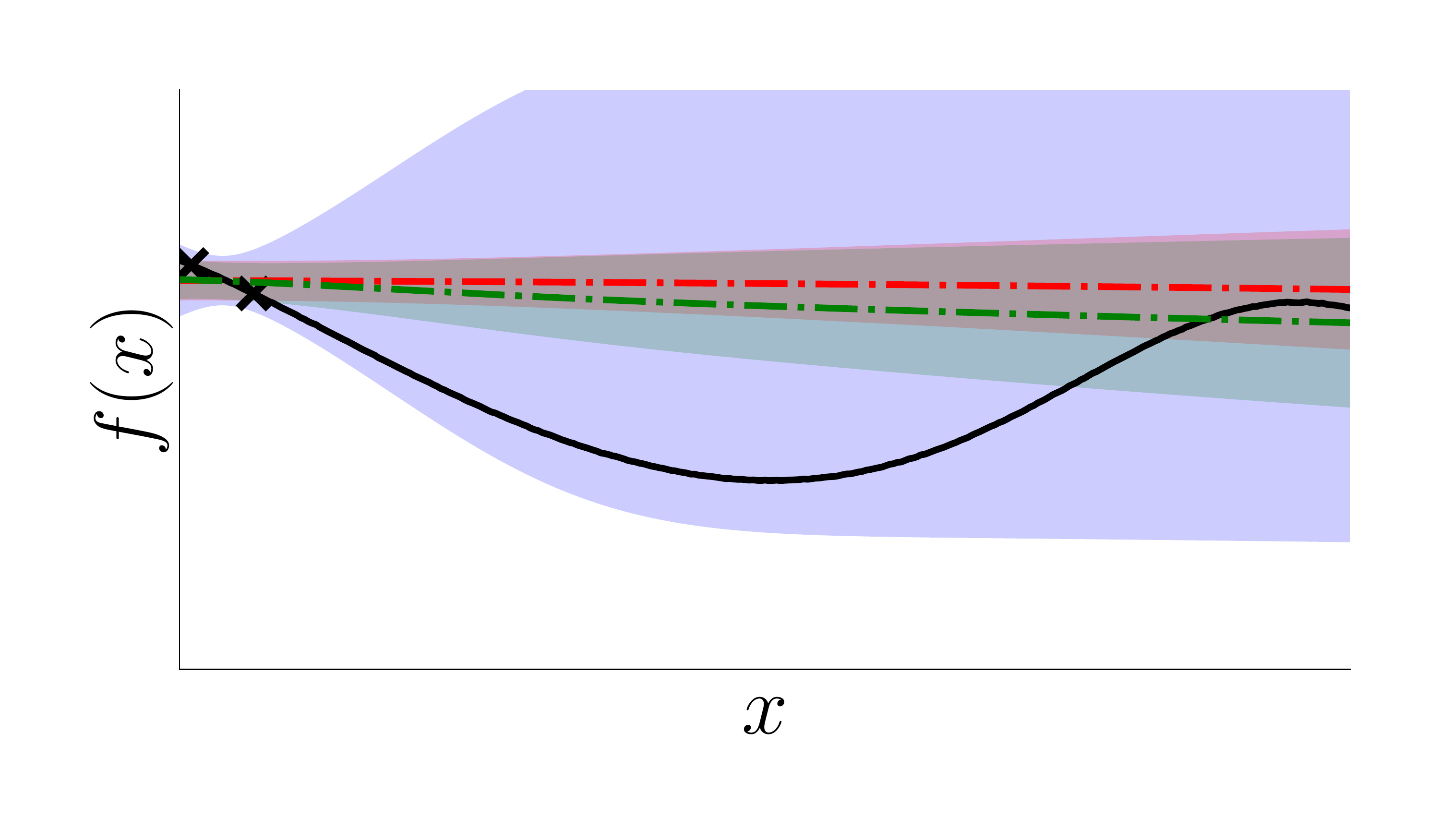}}
	\hfill
	\subfigure[$N=4$ data points.]{\label{fig:robustgpn4}\includegraphics[width=.32\textwidth]{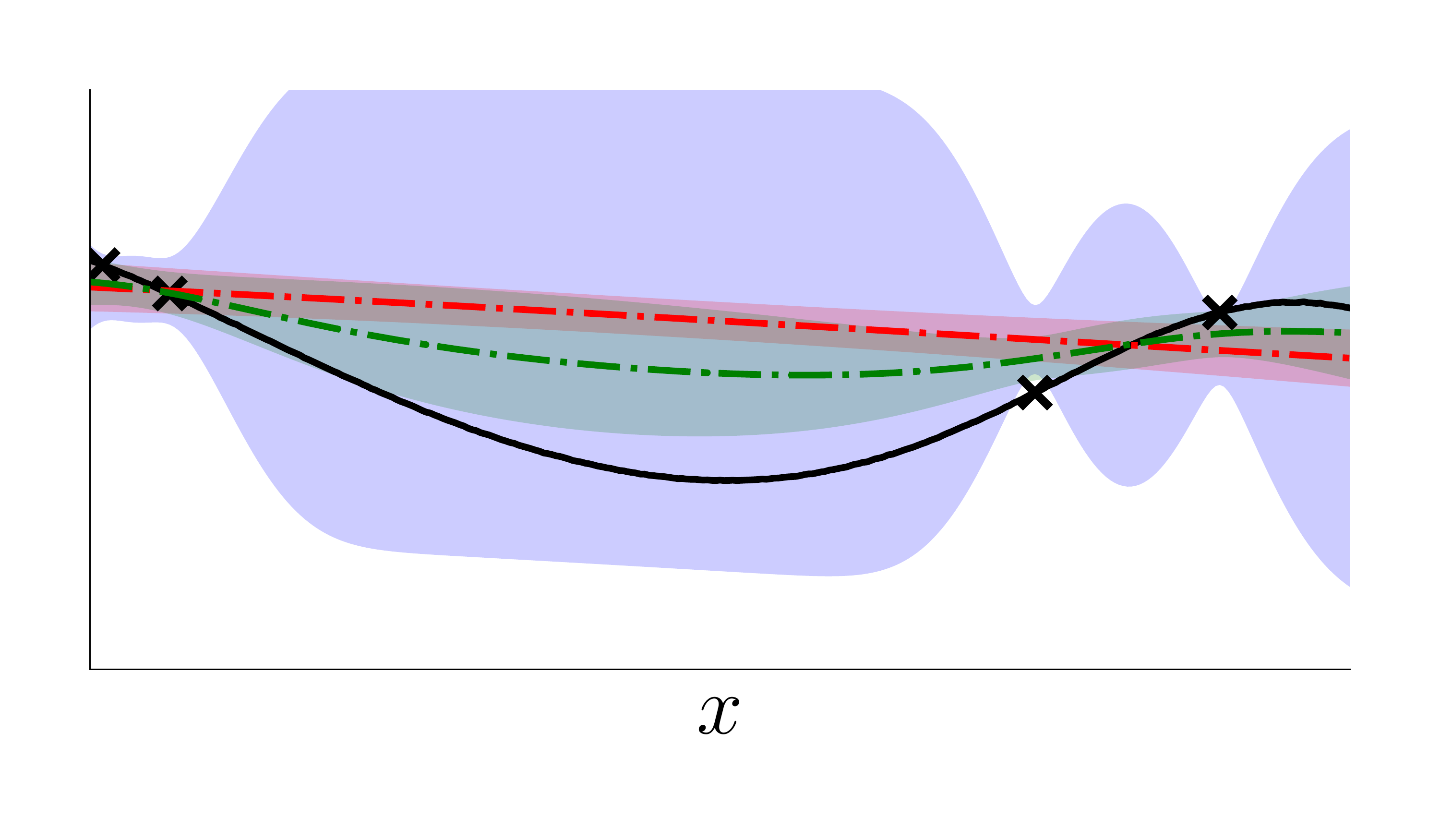}}
	\hfill
	\subfigure[$N=6$ data points.]{\label{fig:robustgpn6}\includegraphics[width=.32\textwidth]{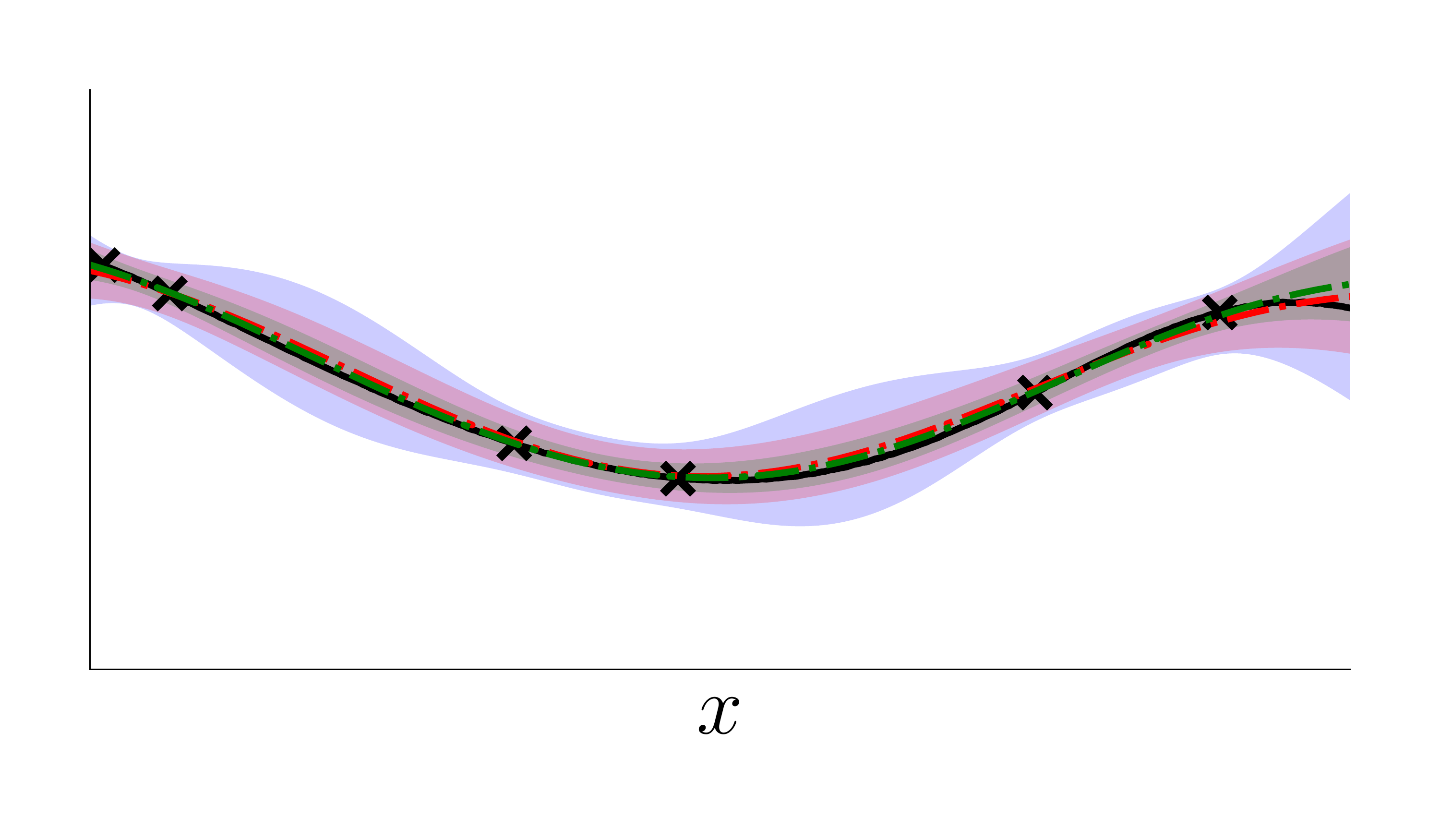}}
	\caption{Function $f(\cdot)$ sampled from a GP. Mean and estimated regression error when using our approach (blue), a vanilla GP (red), and a fully Bayesian GP (green). We used scaling parameters of $\beta = \bar{\beta} = 2$. Both the vanilla GP and the fully Bayesian GP fail to correctly estimate most of the regression error unless the input space is sufficiently covered by the data.} 
	\label{fig:samplefromgp}
\end{figure*}

\begin{table*}[t!]
	\vskip 0.15in
	\begin{center}
		\begin{small}
			\begin{sc}
				\begin{tabular}{lccccccccc}
					\toprule
					Data set& BSTN$_{50}$& BSTN$_{450}$ & ML$_{50}$& ML$_{300}$& Wine$_{200}$ & Wine$_{1000}$ &
					SRCS$_{850}$ &
					SRCS$_{5000}$ & SRCS$_{10000}$  \\
					\midrule
					\midrule
					Dimension& $d=13$& $d=13$& $d=1$ & $d=1$ & $d=11$ &
					$d=11$ & $d=21$ & $d=21$ & $d=21$ \\
					\midrule
					Our approach   &\textbf{0.19}& \textbf{0.35}&
					\textbf{0.00} & \textbf{0.00}& \textbf{0.01} & \textbf{0.01} &
					\textbf{0.08} & 
					\textbf{0.02}   & \textbf{0.01} \\ 
					\midrule
					Vanilla GP    &0.41& 0.48&
					0.11 & {0.01}& 
					0.04 &0.04  & 
					0.98 &
					0.98 & 0.97  \\ 
					\midrule
					Full Bayes    &0.36& 0.44& 
					\textbf{0.00} & \textbf{0.00}& 
					0.04 & 0.04 & 
					0.76 &
					- & -  \\
					\bottomrule
				\end{tabular}
			\end{sc}
		\end{small}
	\end{center}
	\vskip -0.1in
	\caption{Average rate of error bound violations using the proposed approach, vanilla GPs, and fully Bayesian GPs. Lower is better. We set $\bar{\beta}=\beta =2$ to avoid overly conservative error estimates while enabling a fair comparison with high practical relevance. BSTN stands for Boston (house prices), ML for Mauna Loa, and SRCS for Sarcos. The subscripts in the data set names indicate the number of training data points used. For the two largest Sarcos data sets, 
		the number of samples required to get reliable fully Bayesian models is prohibitive, hence we do not provide corresponding predictions. 
		Best-performing methods are in boldface. Our approach performs best across all scenarios.}
	\label{table:performance}
\end{table*}

\noindent We now present experimental results where the performance of the proposed error bound is compared to that of vanilla and fully Bayesian GPs\footnote{The corresponding code can be found at https://github.com/aCapone1/gauss\_proc\_unknown\_hyp.}. We first showcase the prediction error on regression benchmarks, then apply the proposed technique to design a learning-based control law. A Gaussian kernel is used in all cases except the Mauna Loa experiments, and we also consider uncertainty in the signal and noise variances, as discussed in \Cref{section:safegaussianprocesses}. In all except the Sarcos experiments, the fully Bayesian GPs and integral expression in \eqref{eq:optimizationproblem} are approximated using the No-U Turn Sampler algorithm for MCMC \cite{hoffman2014no}. The implementations are carried out using GPyTorch \cite{gardner2018gpytorch}.

The values for $\beta$ proposed in theory are typically very conservative \cite{Srinivas2012}, which makes their use impractical. Similarly, the scaling factor $\bar{\beta}$ proposed by \Cref{theorem:mainresult} assumes very high values. In many safety-critical applications it is common to choose the scaling factor as $\beta = 2$, independently of the hyperparameters $\bm{\vartheta}$ \cite{berkenkamp2017safe,umlauft2017feedback}. Moreover, empirical results indicate that \Cref{lemma:sigmasdecreasewithlengthscale} holds with $\gamma = 1$ in the case of a Gaussian kernel, i.e., the error bound will be at least as robust as that of the vanilla GP for $\beta=\bar{\beta}$. Hence, to obtain a fair, practically relevant and not overly conservative comparison, we set $\bar{\beta} = \beta = 2$ in the following experiments. Note that we still apply the results from \Cref{theorem:mainresult} to compute the posterior variance $\sigma_{\bm{\vartheta}'}(\cdot)$.  
The working hyperparameters $\bm{\vartheta}_0$ are chosen by maximizing the log marginal likelihood \eqref{eq:loglikelihood}. We employ a confidence parameter of $\delta=0.05$ for \eqref{eq:optimizationproblem} and assume uniform distributions as hyperpriors $p(\bm{\vartheta})$. Note that even though this imposes hard bounds on the hyperparameters, these are still very large.

\subsection{Regression - Toy Problem and Benchmarks}

In the regression experiments, to additionally illustrate the behavior of the proposed technique as more data becomes available, we train the GPs with data sets of {varying sizes~$N$}.

We first investigate how the proposed technique performs when estimating the regression error of a function $f(\cdot)$ that is sampled from a GP. The results can be seen in \Cref{fig:samplefromgp}. Our approach always captures the behavior of  the underlying sample $f(\cdot)$. Both the vanilla GP error bound and the fully Bayesian error bound fail to do so if little data is available (\Cref{fig:robustgpn2}) or if the data is sparse (\Cref{fig:robustgpn4}). The estimated error is only accurate for the vanilla and fully Bayesian GPs if the data covers the state space sufficiently well (\Cref{fig:robustgpn6}).

We now apply the proposed technique to estimate the regression error of the Boston house prices data set \cite{scikit-learn}, the UCI wine quality data set \cite{cortez2009modeling}, the Mauna Loa CO$_2$ time series, and the Sarcos data set. For the Mauna Loa data set, we employ a spectral mixture kernel with $20$ mixtures \cite{wilson2013gaussian}, which consists of a sum of $20$ Gaussian kernels multiplied with sinusoidal kernels. It is straightforward to show that its Fourier transform increases with the lengthscales, hence we are able to apply \Cref{theorem:mainresult}. The mixture means, which specify the frequency of the periodic components, are assumed to be fixed except for the fully Bayesian case. For the two largest Sarcos data sets, we perform a Laplace approximation around the maximum of the posterior, which yields a normal distribution \cite{mackay2002information}. To ensure that the corresponding covariance matrix is positive definite, we employ an empirical Bayes approach and specify a quadratic hyperprior around the log likelihood maximum. The bounding hyperparameters are then obtained by computing the corresponding rectangular confidence region as in \citet{vsidak1967rectangular}.

We run each scenario multiple times and select the $N$ training and $N_{\text{test}}$ test points randomly every time. The Sarcos experiments are repeated $10$ times due to high computational requirements, all other experiments are repeated $100$ times. To evaluate performance, we check how often the error bound is violated, which corresponds to measuring the quantity
\begin{equation}
	\label{eq:errorratemetric}
	\frac{1}{N_{\text{test}}}{\sum\limits_{m=1}^{N_{\text{test}}} \mathbb{I}\left( \vert y^{\text{test}}_m - \mu_{\bm{\vartheta}_0}(\bm{x}_m^{\text{test}}) \vert - {\beta}^{\frac{1}{2}}\sigma_{\bm{\vartheta}'}(\bm{x^{\text{test}}_m}) \right)},
\end{equation}
where the superscript test denotes test inputs/outputs, and \[\mathbb{I}(z)=\begin{cases}
	0, & \text{if} \ z\leq 0 \\
	1,& \text{otherwise}
\end{cases}\] is the indicator function. This metric\footnote{ \citet{kuleshov2018accurate} employ a similar metric to determine the calibration error in regression settings.} is highly relevant, since the theoretical guarantees of many GP-based Bayesian optimization and safe control algorithms hinge on the assumption that \eqref{eq:errorratemetric} is equal to zero with high probability \cite{Srinivas2012,berkenkamp2017safe,umlauft2017feedback,Capone2019BacksteppingFP,lederer2019uniform}. In practice, most such algorithms have been shown to provide satisfactory results even if this requirement is not strictly enforced. However, it is still highly desirable for \eqref{eq:errorratemetric} to be as small as possible, as this reduces the amount of potentially unsafe choices.
%

The results are summarized in \Cref{table:performance}. As can be seen, our bound is always more accurate than that of vanilla or fully Bayesian GPs, particularly in low-data regimes. This is to be expected, as the marginal likelihood function is known to be poorly peaked for small $N$ \cite{6790802}. 

\subsection{Control Design with Little Data}

We now apply the proposed technique to design a control law for a safety-critical setting, where a one-link planar manipulator with motor dynamics is to be steered towards the origin using the method presented in \Cref{section:backstepping} and $N=10$ training data points. Note that $N=10$ is not very small for control purposes in the proposed setting, since good performance can already be achieved with as little as $N=50$ data points \cite{Capone2019BacksteppingFP}. The data is obtained using a low-gain sinusoidal input. 

The manipulator dynamics are given by
\begin{align*} D\ddot{\varphi}  + B\dot{\varphi} + G \sin(\varphi) =& \tau \\
	M + \dot{\tau} + H \tau + Z \dot{\varphi} = & u,
\end{align*}
where $\varphi$ and $\tau$ are the system's angle and torque, respectively, $u$ is the motor voltage, which we can control directly, and $D$, $B$, $G$, $M$, $H$, $Z$ are system parameters. By approximating the differential equation of each state $\varphi$, $\dot{\varphi}$, $\tau$ using a separate Gaussian process, we are able to employ a backstepping technique to track a desired trajectory, as described in \Cref{section:backstepping}. We then use \Cref{theorem:backsteppingcontrol} to choose the control gains $C_i(\bm{x})$, $i=1,\ldots,3$, by setting the desired error to $\epsilon_{\max} =  1$. 


To evaluate the control performance, we run $100$ simulations for each setting, where the initial conditions of the manipulator are randomly sampled from a normal distribution. The norm of the resulting tracking error is displayed in \Cref{fig:trackingerrornorm}. As can be seen, the gains obtained using the proposed approach perform considerably better than the ones obtained with vanilla and fully Bayesian GPs. This is because the posterior variance obtained with our approach increases more rapidly away from the collected data points than in the other settings, expressing less confidence in the posterior mean. 


\begin{figure}[t!]
	\centering
	\includegraphics[scale=0.285,trim={0 1.2cm 0 0},clip]{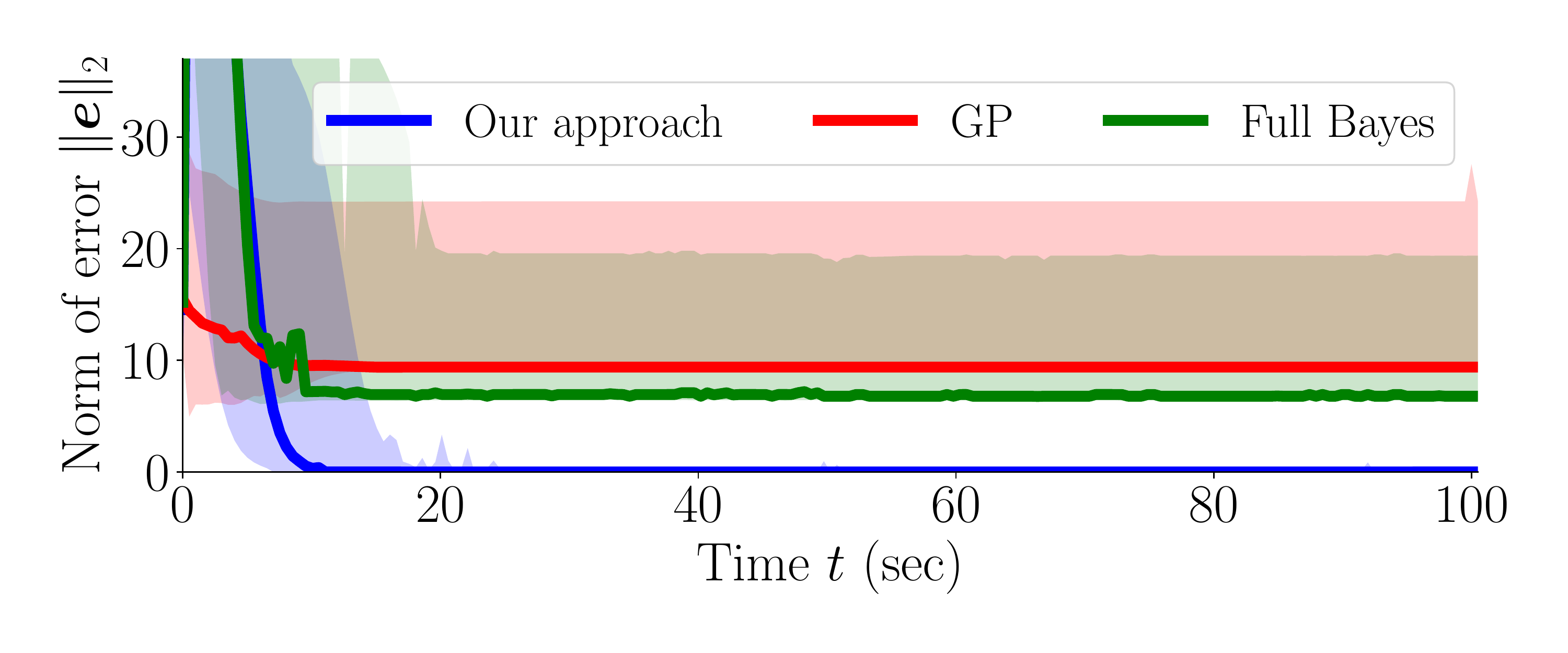}
	\caption{Tracking error norm $\Vert \bm{e}\Vert_2 $ over time $t$. Solid lines and shaded regions correspond to the median and lower/upper deciles, respectively. Blue corresponds to our method, red to vanilla GPs, green to fully Bayesian GPs. Our method yields the desired tracking error, as expected from \Cref{theorem:backsteppingcontrol}, whereas the vanilla and fully Bayesian GPs do not.} 
	\label{fig:trackingerrornorm}
\end{figure}
\section{Conclusion}
\label{section:conclusion}
\noindent We have presented robust uniform error bounds for Gaussian processes with unknown hyperparameters. Our approach is applicable for stationary radially non-decreasing kernels, which are commonly employed in practice. It hinges on computing a confidence region for the hyperparameters, which does not require pre-specified bounds for the hyperparameters. The presented theoretical results make them flexible and easily applicable to safety-critical scenarios, where theoretical guarantees are often required. In numerical regression benchmarks, the proposed error bound was shown to outperform  the error bound obtained with standard and fully Bayesian Gaussian processes. Furthermore, the presented tool resulted in better performance in a control problem, indicating better suitability for safety-critical settings.

\section*{Acknowledgements}

This work was supported in part by  the European Research Council Consolidator Grant Safe data-driven
control for human-centric systems (CO-MAN) under grant agreement number 864686.

We thank Christian Fiedler for the useful comments and constructive feedback on the manuscript.
%

\bibliography{ms}
\bibliographystyle{icml2022}

\newpage
\appendix
\onecolumn

\input{appendix}


\end{document}

%% file: appendix.tex
	\section{Proofs}
	We begin by listing some well-known results required to prove the main statements from the paper. 
	\begin{lemma}[{\citealp[Theorem 6.11]{Wendland2004}}]
		\label{theorem:wendland}
		Suppose that $k \in L_1(\mathbb{R}^d)$ is a continuous function. Then $k(\cdot)$ is positive definite if and only if its Fourier transform is nonnegative and nonvanishing.
	\end{lemma}
	\begin{lemma}[{\citealp[p.70]{rao1973linear}}]
		\label{lemma:invminusinvisps}
		Let $\bm{A},\bm{B} \in \mathbb{R}^{d\times d}$ be two symmetric positive-definite matrices, such that $\bm{A} -\bm{B}$ is symmetric positive-definite. Then $\bm{B}^{-1} - \bm{A}^{-1}$ is symmetric positive-definite.
	\end{lemma}
The following result is a direct consequence of the block matrix inversion formula, which can be found, e.g., in \citealp[p.18]{horn2012matrix}.
	\begin{lemma}[{\citealp[p.18]{horn2012matrix}}]
		\label{lemma:blockmatrixinverse}
		Let 
		\[\bm{K} = \begin{pmatrix}
			\tilde{\bm{K}} & \bm{k} \\ 
			\bm{k}^{\top} & k
		\end{pmatrix}\]
	be an $N\times N$ non-singular partitioned matrix. Then the lower-right entry of its inverse is given by
	\[ [\bm{K}^{-1}]_{N,N} =  \left(k -  \bm{k}^{\top} \tilde{\bm{K}}^{-1} \bm{k}\right)^{-1}. \]
	\end{lemma}
	By putting \Cref{lemma:invminusinvisps} and \Cref{lemma:blockmatrixinverse} together, we obtain the following statement.
	\begin{lemma}
		\label{lemma:covarianceinequality}
		Let 
		\[\bm{K}_1 = \begin{pmatrix}
			\tilde{\bm{K}}_1 & \bm{k}_1 \\ 
			\bm{k}^{\top}_1 & k_1
		\end{pmatrix}, \quad \bm{K}_2 = \begin{pmatrix}
		\tilde{\bm{K}}_2 & \bm{k}_2 \\ 
		\bm{k}^{\top}_2 & k_2
	\end{pmatrix}\]
		be $N\times N$ symmetric positive definite partitioned matrices, such that $\bm{K}_1 - \bm{K}_2$ is also positive definite. Then
		\[  k_1 -  \bm{k}^{\top}_1 \tilde{\bm{K}}^{-1}_1 \bm{k}_1 \geq k_2 -  \bm{k}^{\top}_2 \tilde{\bm{K}}^{-1}_2 \bm{k}_2. \]
	\end{lemma}
\begin{proof}
	From \Cref{lemma:invminusinvisps}, we obtain that $\bm{q}^{\top}\bm{K}_2^{-1}\bm{q} \geq \bm{q}^{\top}\bm{K}_1^{-1}\bm{q}$ holds for any $\bm{q}\in\mathbb{R}^N$. In particular, for $\bm{q} = (0, \cdots 0, 1)^{\top}$, this implies, together with \Cref{lemma:blockmatrixinverse},
	\begin{align*}
		&(k_2 -  \bm{k}^{\top}_2 \tilde{\bm{K}}^{-1}_2 \bm{k}_2)^{-1} = \bm{q}^{\top}\bm{K}_2^{-1} \bm{q}  \geq \bm{q}^{\top}\bm{K}_1^{-1} \bm{q} = (k_1 -  \bm{k}^{\top}_1 \tilde{\bm{K}}^{-1}_1 \bm{k}_1)^{-1}, 
	\end{align*}
i.e., $ k_1 -  \bm{k}^{\top}_1 \tilde{\bm{K}}^{-1}_1 \bm{k}_1 \geq k_2 -  \bm{k}^{\top}_2 \tilde{\bm{K}}^{-1}_2 \bm{k}_2$ holds.
\end{proof}
	Using \Cref{theorem:wendland}, we then obtain the following.
	
	\begin{lemma}
		\label{lemma:kminusprodkispsd}
		Let $\bm{\vartheta}'',\bm{\vartheta} , \bm{\vartheta}'$ be lengthscales with $\bm{\vartheta}''>\bm{\vartheta} > \bm{\vartheta}'$, and let $\gamma = \left(\prod_{i=1}^{d} \frac{{\vartheta}''_i}{\vartheta'_i}\right)^{\frac{1}{2}}$. Furthermore, for an arbitrary measurement data set $\mathcal{D}$, let $\bm{K}_{\bm{\vartheta}}$ and $\bm{K}_{\bm{\vartheta}'}$ denote the corresponding covariance matrices computed using the kernels $k_{\bm{\vartheta}}(\cdot,\cdot)$ and $k_{\bm{\vartheta}'}(\cdot,\cdot)$, respectively. Then the matrix
		\[ {\gamma}^2\bm{K}_{\bm{\vartheta}'}  - \bm{K}_{\bm{\vartheta}}   \]
		is positive semi-definite.
	\end{lemma}
	\begin{proof}
		Recall that, by \Cref{assumption:kernelassumption}, the kernels $k_{\bm{\vartheta}}(\cdot,\cdot)$ and $k_{\bm{\vartheta}'}(\cdot,\cdot)$ are of the form
		\[k_{\bm{\vartheta}}(\bm{x},\bm{x}') = k\left(\left(\frac{x_1-x_1'}{\vartheta_{1}}, \ldots, \frac{x_d-x_d'}{\vartheta_{d}}\right)^{\top}\right), \qquad \qquad k_{\bm{\vartheta}'}(\bm{x},\bm{x}') = k\left(\left(\frac{x_1-x_1'}{\vartheta_{1}'}, \ldots, \frac{x_d-x_d'}{\vartheta_{d}'}\right)^{\top}\right),\]
		respectively, where $k(\cdot)$ has Fourier transform $\hat{k}(\bm{\omega})  \coloneqq \kappa(\lVert \bm{\omega}\rVert_2)$. Define the matrices $\bm{T} \coloneqq \text{diag}(\frac{1}{\vartheta_1}, \ldots,\frac{1}{\vartheta_d})$ and $\bm{T}' \coloneqq \text{diag}(\frac{1}{\vartheta_1'}, \ldots,\frac{1}{\vartheta_d'})$, and the functions $k_{\bm{T}}:\mathbb{R}^d\rightarrow\mathbb{R}$ and $k_{\bm{T}'}:\mathbb{R}^d\rightarrow\mathbb{R}$, such that $k_{\bm{T}}(\bm{x})\coloneqq k(\bm{T}\bm{x}) $ and $k_{\bm{T}'}(\bm{x})\coloneqq k(\bm{T}'\bm{x}) $ holds. Let $\hat{k}_{\bm{T}}(\cdot)$ and $\hat{k}_{\bm{T}'}(\cdot)$ denote the Fourier transform of of $k_{\bm{T}}(\cdot)$ and $k_{\bm{T}'}(\cdot)$, respectively. We then have
		\begin{align} \begin{split}
		\label{eq:fouriertransform}
		\hat{k}_{\bm{T}}(\bm{\omega}) &=  \int\limits_{\mathbb{R}^d} k(\bm{T}\bm{x})  e^{-i 2\pi \bm{x}^\top \bm{\omega}}d\bm{x} \\ & 
		{=}  \int\limits_{\mathbb{R}^d} k(\bm{z}) e^{-i 2\pi \bm{z}^\top\bm{T}^{-\top} \bm{\omega}} {\det(\bm{T}^{-1})}d{\bm{z}}
		\\ &= \left(\prod\limits_{i=1}^d \vartheta_i \right)\int\limits_{\mathbb{R}^d} k(\bm{z}) e^{-i 2\pi \bm{z}^\top\bm{T}^{-\top} \bm{\omega}} d{\bm{z}}
		\\ &  =
		  \left(\prod\limits_{i=1}^d \vartheta_i\right) \hat{k}(\bm{T}^{-1} \bm{\omega})  
		{=} 
		  \left(\prod\limits_{i=1}^d \vartheta_i\right) \kappa(\lVert \bm{T}^{-1} \bm{\omega} \rVert_2). 
		\end{split}\end{align}
		Since $\kappa(\cdot)$ is non-increasing and $\bm{\vartheta} > \bm{\vartheta}'$, we have that
		\[\kappa(\lVert \left(\bm{T}'\right)^{-1} \bm{\omega} \rVert_2) \geq \kappa(\lVert \bm{T}^{-1} \bm{\omega} \rVert_2).\]
		Hence, \eqref{eq:fouriertransform} together with $\gamma^2 > \prod\limits_{i=1}^d\frac{{\vartheta}_i}{{\vartheta}_i'}$ implies \[\gamma^2\hat{k}_{\bm{T} '}(\bm{\omega}) -  \hat{k}_{\bm{T} }(\bm{\omega})\geq \prod\limits_{i=1}^d\frac{{\vartheta}_i}{{\vartheta}_i'}\hat{k}_{\bm{T} '}(\bm{\omega}) -  \hat{k}_{\bm{T} }(\bm{\omega}) \geq 0\] for all $\bm{\omega} \in \mathbb{R}^d$. It then follows from \Cref{theorem:wendland} that $\gamma^2{k}_{\bm{T} '}(\cdot) -  {k}_{\bm{T} }(\cdot)$ is a positive-definite function. Since 
		 \[\gamma^2{k}_{\bm{T} '}(\bm{x}-\bm{x}') -  {k}_{\bm{T} }(\bm{x}-\bm{x}') = \gamma^2k_{\bm{\vartheta}'}(\bm{x},\bm{x}') - k_{\bm{\vartheta}}(\bm{x},\bm{x}')\]
		 this implies the desired result.
	\end{proof}
	\begin{proof}[Proof of \Cref{lemma:sigmasdecreasewithlengthscale}]
		The result follows directly from \Cref{lemma:covarianceinequality} and \Cref{lemma:kminusprodkispsd}.
	\end{proof}

	In order to prove \Cref{theorem:errorwithknownbounds}, we aim to bound the difference between posterior means $\vert \mu_{\bm{\vartheta}}(\bm{x}) - \mu_{\bm{\vartheta}'}(\bm{x}) \vert$. To this end, we employ the two following results.
\begin{lemma}
	\label{lemma:srinivas}
	Let $\rVert \cdot \lVert_{k_{\bm{\vartheta}}}$ denote the RKHS norm with respect to a kernel $k_{\bm{\vartheta}}$, and consider a function $\mu(\cdot) $ with bounded RKHS norm $\rVert \mu \lVert_{k_{\bm{\vartheta}}} < \infty$. Then, for all $\bm{x} \in \X$,
	\begin{align*}
		\begin{split}
			\vert \mu(\bm{x}) \vert^2 \leq   \sigma^2_{\bm{\vartheta}} (\bm{x}) \Bigg(\rVert\mu \rVert^2_{k_{\bm{\vartheta}}} +  \sum \limits_{i=1}^N \left(\frac{\mu(\bm{x}_i)}{\sigma_n}\right)^2\Bigg),
		\end{split}
	\end{align*}
	holds, where $\sigma_n$ is the noise variance and $\bm{x}_1,\ldots,\bm{x}_N$ are the measurement inputs.
\end{lemma}
\begin{proof}
	See \citealp[Appendix B, eq. (11)]{Srinivas2012}.
\end{proof}

	\begin{lemma}[{\citealp{bull2011convergence}}]
		\label{lemma:bull}
		Let $\mu(\cdot)$ be a function with bounded reproducing kernel Hilbert space norm $\rVert \mu \lVert_{k_{\bm{\vartheta}}}$ with respect to the kernel ${k_{\bm{\vartheta}}}$. Then, for all $\bm{\vartheta}' \leq \bm{\vartheta}$,
	\begin{align*}
		\begin{split}
			\rVert \mu \Vert^2_{k_{\bm{\vartheta}'}}  \leq \prod \limits_{i=1}^d \frac{\vartheta_{i}}{\vartheta_i'}\rVert \mu \Vert^2_{k_{\bm{\vartheta}}} .
		\end{split}
	\end{align*}
	\end{lemma}
	We are now able to bound the difference between the means of two Gaussian processes conditioned on the same data, as described in the next statement.
	\begin{lemma}
		\label{lemma:boundedmeandifference}
		Let \Cref{assumption:kernelassumption} hold, let $\bm{\vartheta}',\bm{\vartheta} '',\bm{\vartheta}, \bm{\vartheta}_0 \in \bm{\varTheta}$ be vectors of lengthscales with $\bm{\vartheta}'\leq \bm{\vartheta}_{0},\bm{\vartheta}\leq \bm{\vartheta}''$, and let $\mu_{\bm{\vartheta}_0}(\bm{x}), \mu_{\bm{\vartheta}}(\bm{x})$ denote GP posterior means conditioned on some measurement data set $\mathcal{D}=(\bm{X},\bm{y})$. Then
		\begin{align*}
			\vert \mu_{\bm{\vartheta}_0}(\bm{x}) - \mu_{\bm{\vartheta}}(\bm{x}) \vert^2 \leq \sigma^2_{\bm{\vartheta}'} (\bm{x}) \gamma^24\frac{ \rVert\bm{y}\rVert^2_2}{\sigma_n^2}
		\end{align*}
		holds for all $\bm{x} \in \X$, where
		$\sigma_n$ is the noise variance, and $\bm{y}$ is the measurement data.
	\end{lemma}
\begin{proof}
	Define $\tilde{\bm{K}}_{\bm{\vartheta}} \coloneqq \bm{K}_{\bm{\vartheta}} + \sigma_n\bm{I}$. From \Cref{lemma:srinivas}, we have
	\begin{align}
		\label{eq:meandiffineq}
	\begin{split}
		\quad\vert \mu_{\bm{\vartheta}_0}(\bm{x}) - \mu_{\bm{\vartheta}}(\bm{x}) \vert^2  \
		{\leq} \ \sigma^2_{\bm{\vartheta}'} (\bm{x}) \Bigg(\rVert\mu_{\bm{\vartheta}_0} - \mu_{\bm{\vartheta}}\rVert^2_{k_{\bm{\vartheta}'}} + \sum \limits_{i=1}^N \frac{\left(\mu_{\bm{\vartheta}}(\bm{x}_i) - \mu_{\bm{\vartheta}_0}(\bm{x}_i)\right)^2}{\sigma_n^2}\Bigg) 
	\end{split}
\end{align}
Recall that, for a function of the form $\mu(\cdot)=\sum_{i=1}^N \alpha_i k_{\bm{\vartheta}}(\cdot,\bm{x}_i)$, its RKHS norm with respect to the kernel $k_{\bm{\vartheta}}(\cdot,\cdot)$ is given by
\[\rVert\mu_{\bm{\vartheta}} \lVert^2_{k_{\bm{\vartheta}}} = \sum_{i=1}^N\sum_{j=1}^N \alpha_i \alpha_j k_{\bm{\vartheta}}(\bm{x}_i,\bm{x}_j).
\]
 Hence, the RKHS norm with respect to ${k_{\bm{\vartheta}}}(\cdot,\cdot)$ of a posterior mean function $\mu_{\bm{\vartheta}}(\cdot)$ can be upper-bounded as
\begin{align*}
	\begin{split}
		\rVert\mu_{\bm{\vartheta}} \lVert^2_{k_{\bm{\vartheta}}} = \bm{y}^{\top} \tilde{\bm{K}}_{\bm{\vartheta}}^{-1}\bm{K}_{\bm{\vartheta}}\tilde{\bm{K}}_{\bm{\vartheta}}^{-1}\bm{y} \leq \bm{y}^{\top}\tilde{\bm{K}}_{\bm{\vartheta}}^{-1}\bm{y} \leq \frac{\lVert\bm{y}\rVert_2^2}{\sigma_n^2},
	\end{split}
\end{align*}
where the first inequality is due to the eigenvalues of $\tilde{\bm{K}}_{\bm{\vartheta}}^{-\frac{1}{2}}\bm{K}_{\bm{\vartheta}}\tilde{\bm{K}}_{\bm{\vartheta}}^{-\frac{1}{2}}$ being smaller or equal to one, whereas the last inequality holds because the eigenvalues of $\tilde{\bm{K}}_{\bm{\vartheta}}$ are greater or equal to $\sigma_n^2$. Furthermore, the summation in \eqref{eq:meandiffineq} can be bounded as
\begin{align*}
	&\sum \limits_{i=1}^N \frac{\left(\mu_{\bm{\vartheta}}(\bm{x}_i) - \mu_{\bm{\vartheta}_0}(\bm{x}_i)\right)^2}{\sigma_n^2}  = \frac{1}{\sigma_n^2}\lVert \bm{K}_{\bm{\vartheta}}\tilde{\bm{K}}^{-1}_{\bm{\vartheta}} \bm{y} -\bm{K}_{\bm{\vartheta}_0}\tilde{\bm{K}}^{-1}_{\bm{\vartheta}_0} \bm{y}\rVert^2_2 \\
	\leq & \frac{1}{\sigma_n^2}\lVert \bm{K}_{\bm{\vartheta}}\tilde{\bm{K}}^{-1}_{\bm{\vartheta}} \bm{y}\rVert^2_2 + \frac{1}{\sigma_n^2} \rVert\bm{K}_{\bm{\vartheta}_0}\tilde{\bm{K}}^{-1}_{\bm{\vartheta}_0} \bm{y}\rVert^2_2 \leq \frac{2 \rVert\bm{y}\rVert^2_2}{\sigma_n^2}.
\end{align*}
Plugging this into \eqref{eq:meandiffineq} together with \Cref{lemma:bull,lemma:bull} and $\gamma>1$ yields
\begin{align*}
	\begin{split}
				\quad&\vert \mu_{\bm{\vartheta}_0}(\bm{x}) - \mu_{\bm{\vartheta}}(\bm{x}) \vert^2  
		\\ 
		{\leq} &\sigma^2_{\bm{\vartheta}'} (\bm{x}) \Bigg(\rVert\mu_{\bm{\vartheta}_0} - \mu_{\bm{\vartheta}}\rVert^2_{k_{\bm{\vartheta}'}} + \sum \limits_{i=1}^N \frac{\left(\mu_{\bm{\vartheta}}(\bm{x}_i) - \mu_{\bm{\vartheta}_0}(\bm{x}_i)\right)^2}{\sigma_n^2}\Bigg)  
		\\ 
		{\leq} &\sigma^2_{\bm{\vartheta}'} (\bm{x}) \Bigg(\rVert\mu_{\bm{\vartheta}_0}\rVert^2_{k_{\bm{\vartheta}'}} + \lVert \mu_{\bm{\vartheta}}\rVert^2_{k_{\bm{\vartheta}'}} + \frac{2 \rVert\bm{y}\rVert^2_2}{\sigma_n^2}\Bigg) \\
		{\leq} &\sigma^2_{\bm{\vartheta}'} (\bm{x}) \Bigg(\prod \limits_{i=1}^d \frac{\vartheta_{0,i}}{\vartheta_i'} \rVert\mu_{\bm{\vartheta}_0}\rVert^2_{k_{\bm{\vartheta}_0}} + \prod \limits_{i=1}^d \frac{\vartheta_{i}}{\vartheta_i'}\lVert \mu_{\bm{\vartheta}}\rVert^2_{k_{\bm{\vartheta}}} + \frac{2 \rVert\bm{y}\rVert^2_2}{\sigma_n^2}\Bigg) \\
		{\leq} &\sigma^2_{\bm{\vartheta}'} (\bm{x}) \Bigg(\gamma^2 \rVert\mu_{\bm{\vartheta}_0}\rVert^2_{k_{\bm{\vartheta}_0}} + \gamma^2\lVert \mu_{\bm{\vartheta}}\rVert^2_{k_{\bm{\vartheta}}} + \frac{2 \rVert\bm{y}\rVert^2_2}{\sigma_n^2}\Bigg) \\
		\leq & \sigma^2_{\bm{\vartheta}'} (\bm{x}) \gamma^24\frac{ \rVert\bm{y}\rVert^2_2}{\sigma_n^2}. 
	\end{split}
\end{align*}
\end{proof}
\begin{proof}[Proof of \Cref{theorem:errorwithknownbounds}]
	By applying \Cref{assumption:boundingfunctionbayes,lemma:sigmasdecreasewithlengthscale,lemma:boundedmeandifference}, we obtain
	\begin{align*}
		& \vert f(\bm{x}) - \mu_{\bm{\vartheta}_{0}}(\bm{x})\vert \\ \leq& \vert f(\bm{x}) - \mu_{\bm{\vartheta}}(\bm{x})\vert + \vert \mu_{\bm{\vartheta}} (\bm{x})  - \mu_{\bm{\vartheta}_{0}}(\bm{x})\vert \\ \leq & \left(\max\limits_{\bm{\vartheta}'\leq\bm{\vartheta}\leq \bm{\vartheta}''}\beta^{\frac{1}{2}}(\bm{\vartheta}) \right) \gamma\sigma_{\bm{\vartheta}'}(\bm{x}) +  \frac{2 \rVert\bm{y}\rVert_2}{\sigma_n} \gamma\sigma_{\bm{\vartheta}'}(\bm{x}) \\ = & \bar{\beta}^{\frac{1}{2}} \sigma_{\bm{\vartheta}'}(\bm{x}).
	\end{align*}.
\end{proof}
%

	\begin{proof}[Proof of \Cref{theorem:mainresult}]
		It follows from the definition of $\mathcal{P}_\delta$ that $\bm{\vartheta}'\leq \bm{\vartheta}\leq \bm{\vartheta}''$ holds with probability at least $(1-\delta)$. \Cref{theorem:mainresult} is then a direct consequence of \Cref{theorem:errorwithknownbounds}.
	\end{proof}
	
	\begin{proof}[Proof of \Cref{theorem:backsteppingcontrol}]
		The dynamics of the control error can be written in the compact form \citep{Capone2019BacksteppingFP}
		\begin{align}
			\dot{\bm{e}} =  \dot{\bm{x}} - \dot{\bm{x}}_d = \Delta \bm{f} - \left(\bm{C} +\bm{G}\right)\bm{e},
  		\end{align}
  		where $\Delta \bm{f} \coloneqq (f_1 -\mu_1, \ldots, f_d -\mu_d)$, $\bm{C}\coloneqq \text{diag}(C_1(\bm{x},\ldots,C_d(\bm{x})))$, and \begin{align*}
  		\bm{G} \coloneqq \begin{pmatrix}
  		0 & g_1 & 0 & \cdots & 0 \\
  		-g_1 & 0 & g_2 & & \vdots \\
  		0 & -g_2 & 0 & &  \\
  		\vdots &  &  & \ddots & g_{d-1}  \\
  		0 & \cdots & & -g_{d-1} & 0
  		\end{pmatrix}.
  		\end{align*}
  		Consider then the Lyapunov function $V(\bm{e}) = \frac{1}{2} \bm{e}^{\top} \bm{e}$. Due to \Cref{theorem:mainresult}, the corresponding time derivative is bounded with probability $(1-\delta)^m(1-\rho)^m$ as
  		\begin{align}
  			\label{eq:timederlyapcontrol}
  			&\dot{V}(\bm{e}) = {\bm{e}}^{\top}\dot{\bm{e}}  = {\bm{e}}^{\top}\Delta\bm{f} - \bm{e}^{\top}\bm{C}\bm{e} \\ \leq& \lVert \bm{e} \rVert_2 \left(\lVert \Delta\bm{f} \rVert_2 - \lVert \bm{e} \rVert_2 \min\limits_{i}C_i(\bm{x}) \right) \\
  			{\leq}& \lVert \bm{e} \rVert_2 C_i(\bm{x}) \left(\xi_{\text{des}} - \lVert \bm{e} \rVert_2 \right),                
  		\end{align}
  		where we employ $\bm{e}^{\top}\bm{G}\bm{e}= \bm{0}$. The term $\lVert \bm{e} \rVert_2 C_i(\bm{x}) \left(\xi_{\text{des}} - \lVert \bm{e} \rVert_2\right)$ is negative for $\lVert \bm{e} \rVert_2 < \xi_{\text{des}}$, hence \eqref{eq:timederlyapcontrol} implies that the norm of $\bm{e}$ does not grow beyond $\xi_{\text{des}}$ after a specific time $T$ \cite{Khalil1996} with probability $(1-\delta)^m(1-\rho)^m$.
	\end{proof}

	\section{Experiments}
	All experiments were carried out using an AMD Ryzen Threadripper 2990WX with 32 cores. 
	
	The hyperpriors in all cases except the Sarcos experiments with more than $N=850$ training inputs are uniform distributions. The corresponding upper and lower bounds are shown in \Cref{table:hyperpriors}.

	\begin{table*}[t!]
		\vskip 0.15in
		\begin{center}
			\begin{small}
				\begin{sc}
					\begin{tabular}{lccccc}
						\toprule
						Experiment& BSTN & ML &Wine& SRCS& Control   \\
						\midrule
						\midrule
						Signal variance& $[1,50]$& $[10^{-10},10^5]$& $[10^{-2},10^2]$ & $[10^{-1},10^3]$ & $[10^{-6},10]$  \\
						\midrule
						Lengthscales &$[10^{-1},10^2]$& $[10^{-10},5\times10^{12}]$&
						$[10^{-2},10]$ & $[10^{-1},50]$& $[10^{-15},10^{-2}]$ \\ 
						\midrule
						Noise variance    &$[10^{-1},10^2]$& $[10^{-5},10^2]$&
						$[10^{-2},1]$ & $[10^{-2},80]$& 
						$[10^{-5},10^{-1}]$  \\
						\bottomrule
					\end{tabular}
				\end{sc}
			\end{small}
		\end{center}
		\vskip -0.1in
		\caption{Lower/upper bounds of uniform distributions used as hyperpriors. BSTN stands for Boston (house prices), ML for Mauna Loa, and SRCS for Sarcos. Control refers to the backstepping control experiment in \Cref{section:results}.}
		\label{table:hyperpriors}
	\end{table*}
	 \subsection{Setting $\bar{\beta} = \beta$}
	 \label{subsection:settingbeta}
	 In the experimental section, we employ $\bar{\beta} = \beta$, as opposed to explicitly computing the values suggested by \Cref{theorem:mainresult}. We now provide some justification for this choice. 
	 
	 Recall that $\bar{\beta}^{\frac{1}{2}}$ in \Cref{theorem:mainresult} is computed as
	 \begin{align*}
	 	\bar{\beta}^{\frac{1}{2}} = \gamma\left(\max\limits_{\bm{\vartheta}'\leq\bm{\vartheta}\leq \bm{\vartheta}''}\beta^{\frac{1}{2}}(\bm{\vartheta})  +  \frac{2 \rVert\bm{y}\rVert_2}{\sigma_n} \right).
	 \end{align*}
	 As can be seen in the proof of \Cref{theorem:errorwithknownbounds}, the term $2\sigma_n^{-1} \rVert\bm{y}\rVert_2$ is used to upper-bound the discrepancy $\vert \mu_{\bm{\vartheta}} (\bm{x})  - \mu_{\bm{\vartheta}_{0}}(\bm{x})\vert$ between the posterior mean given lengthscales $\bm{\vartheta}$ sampled from the posterior and that of the working hyperparameters $\bm{\vartheta}_{0}$. Hence, if less conservative bounds for the difference $\vert \mu_{\bm{\vartheta}} (\bm{x})  - \mu_{\bm{\vartheta}_{0}}(\bm{x})\vert$ are available, they can be used to replace the estimate $2\sigma_n^{-1} \rVert\bm{y}\rVert_2$ without losing theoretical guarantees. In our experiments, we observed that the difference $\vert \mu_{\bm{\vartheta}} (\bm{x})  - \mu_{\bm{\vartheta}_{0}}(\bm{x})\vert$ was often small, and that \Cref{lemma:covarianceinequality} always held with $\gamma=1$, which in turn suggests that, in the case of Gaussian kernels, the posterior variance is decreasing with respect to the lengthscales. Furthermore, we have
	 \[\max\limits_{\bm{\vartheta}' \leq \bm{\vartheta} \leq \bm{\vartheta}''} \beta^{\frac{1}{2}}(\bm{\vartheta}) = \sqrt{2}\]
	 for all $\bm{\vartheta}$ by assumption. Hence, setting $\bar{\beta} = \beta$, corresponds to ignoring the discrepancy term $\vert \mu_{\bm{\vartheta}} (\bm{x})  - \mu_{\bm{\vartheta}_{0}}(\bm{x})\vert$ and choosing $\gamma=1$. We thus obtain the uniform error bound
	 \begin{align*}
	 	\begin{split}
	 		\vert f(\bm{x}) - \mu_{\bm{\vartheta}_0}(\bm{x}) \vert \leq   {\beta}^{\frac{1}{2}}\sigma_{\bm{\vartheta}'}(\bm{x}).
	 	\end{split}
	 \end{align*}
	 We stress that this upper bound will still be more conservative than that obtained with the working hyperparameters $\bm{\vartheta}_0$. 
	 
	 	 \subsection{Laplace Approximation and Empirical Bayes}
	 We now detail the Laplace approximation and empirical Bayes approach used for the two largest Sarcos data sets ($N=5000$ and $N=10000$).
	 
	 Typically, the Laplace approximation of the posterior is obtained by computing the Hessian of the posterior around its maximum, and then treating the negative Hessian, which should be positive definite, as the covariance matrix of the corresponding Gaussian approximation \cite{mackay2002information}. 
	 
	 In our setting, using uniform hyperpriors would lead to the Hessian of the posterior being the Hessian of the log likelihood, since the derivatives of the uniform hyperprior are zero almost everywhere. However, the Hessian of the log likelihood is not always negative definite due to numerical issues or premature termination of the optimization algorithm. Moreover, in high-dimensional settings such as the Sarcos data set, it is often poorly peaked, even for large $N$. Since the resulting normal distribution extends to the domain with negative lengthscales, as opposed to the exact posterior, the corresponding confidence region can be exceedingly conservative. For these reasons, we consider an empirical Bayes approach, where the prior is chosen based on the data \cite{robbins2020empirical}. In our case, this is achieved by specifying a normal hyperprior around the (estimated) maximum of the log likelihood $\bm{\vartheta}_0$, i.e.,
	 \begin{align*}
	 	\log p(\bm{\vartheta}) \hat{=} -\left(\bm{\vartheta} - \bm{\vartheta}_0\right)^{\top} \bm{H}_p \left(\bm{\vartheta} - \bm{\vartheta}_0\right),
	 \end{align*} 
	where $\hat{=}$ denotes equality up to a constant, and $\bm{H}_p = h_p \bm{I}$ is a diagonal matrix such that the Hessian of the log posterior
	\begin{align*}
				\bm{H}_{\bm{\vartheta}} \left(\log(p(\bm{\vartheta} \vert \bm{y}, \bm{X}))\right) = -\bm{H}_p + \bm{H}_{\bm{\vartheta}}\left(\log p(\bm{y} \vert \bm{X}, \bm{\vartheta})\right)
	\end{align*}
	is negative definite. Here $\bm{H}_{\bm{\vartheta}}(\cdot)$ denotes the Hessian operator with respect to $\bm{\vartheta}$. Note that increasing $h_p$ results in a less conservative set of bounding hyperparameters $\bm{\vartheta}',\bm{\vartheta}''$, since it implies more confidence in the estimated maximum $\bm{\vartheta}_0$. In the Sarcos experiments, we set $h_p $ to ten times the largest nonnegative eigenvalue of $\bm{H}_{\bm{\vartheta}}\left(\log p(\bm{y} \vert \bm{X}, \bm{\vartheta})\right)$.
	
	\subsection{Control Problem}
	The dimensionless parameters of the dynamical system are given by $M=0.05$, $B=D=1$, $G=Z=10$, $H=0.5$. We employ a command-filtered backstepping approach to bypass the computation of the control input, which has no practical impact on the results if a corresponding filtering parameter is chosen high enough. For more details, see \cite{Farrell2009,Capone2019BacksteppingFP}.